%% file: main.tex
\documentclass{article}
\usepackage{authblk}
\usepackage{mathnotations}
\usepackage{algorithmic}
\usepackage{algorithm}
\usepackage{basicreq}
\usepackage[round]{natbib}

\usepackage[subtle]{savetrees}

\begin{document}

\title{Exploration in Linear Bandits with Rich Action Sets and its Implications for Inference}
\author[1]{Debangshu Banerjee}
\author[2]{Avishek Ghosh}
\author[3]{Sayak Ray Chowdhury}
\author[4]{Aditya Gopalan}
\affil[1,4]{ Department of Electrical and Communication Engineering, Indian Institute of Science, India}
\affil[2]{Halicioglu Data Science Institute, 
University of California, USA}
\affil[3]{Division of Systems Engineering,
Boston University, USA}
\maketitle

\begin{abstract}
  We present a non-asymptotic lower bound on the spectrum of the design matrix generated by any linear bandit algorithm with sub-linear regret when the action set has well-behaved curvature. Specifically, we show that the minimum eigenvalue of the expected design matrix grows as $\Omega(\sqrt{n})$ whenever the expected cumulative regret of the algorithm is $O(\sqrt{n})$, where $n$ is the learning horizon, and the action-space has a constant Hessian around the optimal arm. This shows that such action-spaces force a polynomial lower bound on the least eigenvalue, rather than a logarithmic lower bound  as shown by \citet{lattimore2017end}for discrete (i.e., well-separated) action spaces. Furthermore, while the latter holds only in the asymptotic regime ($n \to \infty$), our result for these ``locally rich" action spaces is any-time. Additionally, under a mild technical assumption, we obtain a similar lower bound on the minimum eigen value holding with high probability. We apply our result to two practical scenarios -- \emph{model selection} and \emph{clustering} in linear bandits. For model selection, we show that an epoch-based linear bandit algorithm adapts to the true model complexity at a rate exponential in the number of epochs, by virtue of our novel spectral bound. For clustering, we consider a multi agent framework where we show, by leveraging the spectral result, that no forced exploration is necessary---the agents can run a linear bandit algorithm and estimate their underlying parameters at once, and hence incur a low regret.
\end{abstract}

\input{introduction}
\input{notations}

\input{key-result}

\input{experiments}
\input{model-selection}

\input{clustering}
\input{conclusion}

\clearpage

\bibliographystyle{plainnat}
\bibliography{References}
\newpage
\onecolumn

\input{app_main}

\input{app_convex}

\input{app_model_selection}
\input{app_cluster}

\input{app_high_probability}

\input{app_tech_lemmas}

\end{document}

%% file: introduction.tex
\section{INTRODUCTION}\label{sec:intro}
Bandit optimisation traditionally focuses on the problem of minimising cumulative {\em regret}, or the shortfall in reward incurred along the trajectory of learning. To this end, it has yielded optimal, low-regret strategies such as UCB and Thompson sampling \citep{oful, abeille2017linear}. 
On the other hand, there is also the important {\em inference} goal in bandit problems, in which the experimenter, having access to arms or alternatives, wishes to infer some useful properties of, or estimate a quantity related to, the system by sequential sampling. Perhaps the most well-known example is the problem of best arm identification in multi-armed bandits \citep{even2006action, soare2014best}, which is essentially a sequential hypothesis testing problem. It is well known that for optimal error rates in identifying the best arm, it is necessary to sample arms with at least {\em constant} frequencies, which is vastly more exploratory than the frequencies required for regret minimization \citep{bubeck2011pure}.

In this paper, we are interested in identifying settings in which both objectives -- sublinear regret and fast inference (estimation error) -- are simultaneously possible, opening the door to many useful applications that combine both reward optimisation and statistical inference. As we shall see, this is possible with standard bandit algorithms such as (linear) UCB provided there is sufficient `local richness' of actions or arms in the problem (think `suitably continuous arm set'), which is often a reasonable structure encountered in many bandit optimization problems with continuous spaces of alternatives (power control, dynamic pricing, etc.)
To introduce ideas, consider  standard linear bandit with additive Gaussian noise. A key trajectory-dependent quantity that connects the two goals of regret minimisation and parameter estimation is the minimum singular value of the design matrix, $\lambda_{\min}(V_n)$. To see why, notice that any typical bandit algorithm for regret minimization forms, either explicitly or otherwise, confidence sets for the unknown linear model parameter $\theta^*$, of the form $\|\theta - \hat{\theta}\|_{V_n} \leq c\sqrt{\ln(n)}$, where $c>0$ is a constant. Thus, getting a lower bound on the growth of $\lambda_{\min}(V_n)$, would help in determining how fast the confidence sets shrink and in return would help to infer the true bandit parameter $\theta^*$. 

The closest related work that sheds light on the singular value of the design matrix is by \citet{lattimore2017end}. The authors show that for a discrete action-space bandit and any bandit algorithm with sub-polynomial regret (e.g., UCB), the minimum singular value of the expected design matrix ($\mathbb{E}[V_n]$ ) must grow at a logarithmic rate over time. However, their analysis holds true only in the asymptotic regime, with no information available on what happens in a finite time horizon.

We show that in action-spaces with ``nice" local curvature properties, bandit algorithms which have inherently good regret properties (at most of the order of $\sqrt{n}$), the minimum singular value of the expected design matrix grows at least as order of $\sqrt{n}$. This is accomplished by a novel use of matrix perturbation techniques (Weyl's inequality and the Davis-Kahan sin-$\theta$ theorem) together with the information-theoretic data-processing inequality. Moreover, this result holds true in a finite-time horizon and not just in the asymptotic regime: for all time horizons $n$ larger than a baseline value $n_0$, we show $\lambda_{\min}\mathbb{E}[V_n] \geq \gamma\sqrt{n}$,
where $\gamma$ is a positive constant which depends upon the local curvature and the algorithm being used.
A key implication of this result is that in bandits with continuous action-spaces, low-regret algorithms also offer significant exploration in the sense of  estimating all `directions' in the parameter space. This not only extends the work of  \citet{lattimore2017end} to the continuous action setting but also strengthens it to a finite time result from an asymptotic one. 

We conclude with two illustrative applications of our theory under a mild assumption that the same $\Omega(\sqrt{n})$ growth of the minimum eigenvalue holds also in high probability.\footnote{In the appendix, we provide a technical condition on the trajectory of linear bandit algorithms under which this holds.} Specifically, we consider the model-selection \citep{foster2019model} and clustering  \citep{clustering_online} problems in linear bandits to apply our result. In the model selection application, the norm of the unknown parameter $\theta$ is viewed as a measure of complexity of the problem. We show that a variant of the well-known Optimism in the Face of Uncertainty (OFU) algorithm can adapt to $\norm{\theta}$. This is achieved by a careful application of our result to control the rate at which norm estimates converge to the  the true norm. It is important to note that a similar result is achieved by \citet{ghosh2021problem} in the different but related setting of \emph{stochastic contextual bandits}, albeit with restrictive assumptions on the contexts. We are able to proceed without such restrictions by virtue of our result. In the clustering setup, we consider several linear bandit agents partitioned into $k$ clusters. We propose a clustering algorithm without any explicit exploration, where the agents simultaneously estimate their (linear model) parameter and attempt to play low-regret actions. When the clusters are separated, our algorithm obtains the correct clustering with high probability. Note that in \citet{gentile2014online,ghosh2021collaborative}, the framework of clustered bandits was considered in a contextual framework with several strong assumptions on the context distribution. Our work demonstrates that similar guarantees are attainable (in the context-free linear bandit setup) without additional assumptions via exploiting the rich-action-set inference result (Theorem~\ref{thm:main-result}).
Finally, we empirically validate that the minimum eigenvalue of the design matrix generated by the well-known Thompson sampling algorithm \citep{agrawal2013thompson} indeed grows at a rate larger than $\sqrt{n}$. 

\textbf{Related work.}
The linear bandit problem has been studied extensively in a large body of work starting from the classic work of \citet{auer2002finite}. In this model, algorithms based on the celebrated optimism in the face of uncertainty principle has been designed and analyzed by several authors \citep{chu2011contextual,dani2008stochastic,oful}. A related approach is posterior sampling, also known as Thompson sampling, where sufficient exploration is achieved by randomly sampling a parameter from a posterior distribution over $\theta$ \citep{agrawal2013thompson,abeille2017linear}. Another related line of work consider the linear reward model in reproducing kernel Hilbert spaces \citep{srinivas2009gaussian,valko2013finite,chowdhury2017kernelized}.
Spectral properties of the expected design matrix under a discrete action-space has been studied by \citet{lattimore2017end}, while \citet{hao2020adaptive} handles the finitely many contextual case with each context having discrete action space. 
The framework of model selection has recently gained a lot of momentum in the bandit literature. For example, \citet{pmlr-v108-chatterji20b} introduced a hypothesis test based framework to select either the standard bandit or the linear bandit model. Furthermore, in \citet{ghosh2021problem},  the authors define parameter norm and sparsity as complexity parameters for stochastic linear bandit and adapt to those without any apriori knowledge, and obtain model selection guarantees. Moreover, \citet{foster2019model} introduces an adaptive algorithm for a similar linear bandit problem with sparsity as a measure of complexity. Additionally, there are different line of works, that uses the the corrall framework of \citet{agarwal2017corralling} to obtain adaptive algorithms for bandits and reinforcement learning (for example, see \citet{pacchiano2020regret}). Very recently, for generic contextual bandits, the adaptation question is also addressed in \citet{krishnamurthy2021optimal}. Furthermore, on clustering of bandits, \citet{gentile2014online} proposes an algorithm that works only when the cluster separation is large, which was further improved in \citet{ghosh2021collaborative}, where near optimal regret is obtained even when the clusters are not separable. Moreover, \citet{ghosh2021collaborative} also proposes a natural personalization framework, which is a generalization of the clustering setup.


%% file: notations.tex
\textbf{Notation.} For a positive definite matrix $G$, (denoted as $G \succ 0$) and vector $x$ we write $\|x\|^2_G = x^\top G x $ The euclidean norm of a vector $x$ is denoted as $\|x\|$ and the spectral norm of a matrix $G$ is $\|G\| = \lambda_{\max}(G)$. For hermitian matrices we assume the eigen-decomposition as $G = \sum_{i=1}^d\lambda_i u_i u_i^\top$, with $\lambda_1 > \lambda_d = \lambda_{\min}$. We use standard definitions of Landau Notation when using $O$, $o$, $\Omega$ and $\omega$ notation. We use $\mathbb{E}_\theta[\;.\;]$, to emphasize that the underlying bandit instance is parameterized by $\theta$.

\textbf{Problem setting.}
We consider the linear bandit model of \citet{oful}. Let $\cX \subset \Real^d$. The learner interact with the environment over $n$ rounds. At each round $t$, the learner chooses an action $A_t \in \cX$ and correspondingly observes a reward $Y_t = \inner{A_t}{\theta} + \eta_t$, where $\eta_t$ is a zero-mean Gaussian noise, and $\theta \in \mathbb{R}^d$ is the unknown parameter. The optimal action is $x^* = \arg \max_{x \in \cX} \langle x, \theta\rangle$. The performance of the learner is typically measured using its expected regret, defined as 
\begin{align*}
    \mathbb{E}[R_n(\theta)] = \expect{\sum\nolimits_{t=1}^n\langle x^* - A_t, \theta\rangle}~.
\end{align*}
Here the expectation is over the action-selection strategy of the learner, denoted, where needed, by $\pi$, and over the randomness in observed rewards.

%% file: key-result.tex
\section{MAIN RESULT}
In this section, we develop our main result which characterizes the growth of the minimum eigenvalue of the design matrix for linear bandit algorithms run on action spaces with suitable `local curvature'. By this we mean action spaces which are hyper-surfaces of the form $\{x : f(x) = c\}$, where $f$ is a twice-continuously differentiable function. The following definition expresses the local curvature property needed for our result. 


\begin{definition}[Locally Constant Hessian (LCH) surface]
\label{def:action}
Consider the action space defined by $\mathcal{X} = \{x \in \mathbb{R}^d : f(x) = c\}$, where $f : \mathbb{R}^d \to \mathbb{R}$ is a $C^2$ function (i.e., all second-order partial derivatives of $f$ exist and are continuous) and $c \in \mathbb{R}$. Let $\theta \in \mathbb{R}^d$. $\mathcal{X}$ is said to be a LCH surface w.r.t. $\theta$ if: (i) there is a unique reward-optimal arm with respect to $\theta$ (denoted by  $\mathrm{OPT}_\cX(\theta) = \arg\max_{x \in \cX}\langle x,\theta\rangle$), and (ii) there is an open neighborhood $U \subset \mathbb{R}^d$ of $\mathrm{OPT}_{\cX}(\theta)$ over which the Hessian of $f$ is constant and positive-definite.

\end{definition}
\textbf{Examples of LCH action spaces.}
Any ellipsoidal action space $\mathcal{E} = \{x \in \mathbb{R}^d: x^\top A^{-1} x = c \}$, with $c > 0$ and $A$ positive-definite, is an LCH action space w.r.t. {\it every} $\theta \in \mathbb{R}^d$, as the Hessian of $f(x) = x^\top A^{-1} x$ is the constant p.d. matrix $A^{-1}$. However, an action space can be LCH  just by being `locally ellipsoidal'. As an example, consider an ellipsoid $\mathcal{E} = \{x \in \mathbb{R}^d: x^\top A^{-1} x = c \}$. Let $\theta$ be a bandit parameter and $x^* = \arg\max_{x \in \mathcal{E}} \langle x , \theta \rangle$ be optimal for $\theta$. For some $\delta > 0$, let $\mathcal{B}_\delta$ be an open $\delta$-ball in $\mathbb{R}^d$ containing $x^*$. Consider an action space $\mathcal{X}$ which coincides with $\mathcal{E}$ in the neighborhood $\mathcal{B}_\delta$ and is arbitrary outside it, i.e., $\mathcal{X} \cap \mathcal{B}_\delta$ = $\mathcal{E} \cap \mathcal{B}_\delta$. It follows that $\mathcal{X}$ is LCH w.r.t. $\theta$.


\begin{theorem}
\label{thm:main-result}
Let the action-space $\cX$ be a Locally Constant Hessian(LCH) surface in $\mathbb{R}^{d-1}$ w.r.t. a bandit parameter $\theta$. Let $\bar{G}_n = \mathbb{E}_\theta\left[\sum_{s=1}^n A_sA_s^\top\right]$, where $A_s$ are arms in $\cX$ drawn according to some bandit algorithm. For any bandit algorithm which suffers expected regret\footnote{Our big-Oh and Omega notations throughout omit polylogarithmic dependencies for ease of presentation.} at most $O(\sqrt{n})$, 
\begin{align*}
 \lambda_{\min}(\bar{G}_n) = \Omega(\sqrt{n})\;.   
\end{align*}
That is, there exists an $n_0$ and a constant $\gamma >0$, such that for all $n \geq n_0$, $\lambda_{\min}(\bar{G}_n) \geq \gamma\sqrt{n}$.
\end{theorem}
The constant $\gamma$ depends upon the condition number of the Hessian, the algorithmic constants hidden by $O()$, and the size of the bandit parameter $\theta$. The constant $n_0$ depends on the algorithmic constants hidden by $O()$, the size of the neighbourbood over which the Hessian is constant, the size of the action domain $\norm{\cX}$, the singular value of the Hessian and the size of the bandit parameter $\theta$. \\

\textbf{Comparison with the result of  \citet{lattimore2017end}.}
The authors show a similar result for asymptotically large $n$. Specifically, they show that for a linear bandit with a {\em discrete} action-space (i.e., one for which the arms' suboptimality gaps are at least a positive constant), for any good (i.e., low-regret) bandit algorithm, it holds that
\begin{align*}
    \lim \inf_{n \to \infty}\frac{\lambda_{\min}(\bar{G}_n)}{\log n} > 0.
\end{align*}
In contrast, we prove a bound for any finite $n > n_0$. Moreover, our result applies to a broad category of action spaces. \\

\textbf{Comparison with \cite{bubeck2011pure}.} \citet{bubeck2011pure} show that for a {\em finite-armed} bandit, an optimal cumulative regret ($O(\log T)$) algorithm must suffer $\Omega(\text{poly}(1/T))$ simple regret, or equivalently, a probability of misidentifying the best arm for inference, in $T$ time slots. Optimal inference (best arm) identification algorithms can, in contrast, obtain $e^{-\Omega(T)}$ simple regret at the cost of linear cumulative regret. Our result does not contradict this paper as our action space is {\em richer} (i.e., continuous) than that of a finite-armed bandit, leading to qualitatively different behavior: (a) On one hand, due to the minimum arm suboptimality gap in this setting being zero, the optimal cumulative regret rate is $O(\sqrt{T})$ (regret minimization is harder), (b) On the other hand, inference is easier with an optimal cumulative regret strategy with estimation error (for estimating $\theta^*$) decaying as $O(T^{-1/4})$. Thus, for action-spaces which are locally "nice" enough as described above, any good regret algorithm must induce a well conditioned expected design matrix, and this has implications for parameter recovery as illustrated later. \\


\textbf{Dependence on dimension.}
Though our result does not explicitly indicate the dependence on the ambient (feature) dimension $d$, it is sensitive to it via the regret of the linear bandit algorithm in question. 
For example, for the OFUL algorithm \citep{oful}, we have a regret bound varying with the dimension as  $O(d)$, while for Thompson Sampling (TS) we have a regret bound depending on $d$ as $O(d^{3/2})$ \citep{agrawal2013thompson, abeille2017linear}. The constants $n_0$ and $\gamma$ of Theorem \ref{thm:main-result} depend on $d$ as $\Omega(d^2)$ and $\Omega(\frac{1}{d})$ for OFUL and as $\Omega(d^3)$ and $\Omega(\frac{1}{d^{3/2}})$ for TS respectively for a spherical action space. We provide more remarks on this in the experiment section.


\subsection{Key Technique: Overview}
In this section, for simplicity and insights, we provide a proof sketch for Theorem \ref{thm:main-result} for the spherical action space $\cX = \cS^{d-1}$. In Appendix~\ref{app:main}, we first generalize these ideas for general ellipsoidal results, and finally prove for Locally Constant Hessian surfaces. 

\begin{proof}[Proof Sketch]
We start with the following information inequality for linear bandits (see Lemma~\ref{lemma:Garivier_Kauffmann} in appendix and \cite{lattimore2020bandit}): 
\begin{align}
\label{eq:info_ineq}
\frac{1}{2}\|\theta - \theta'\|^2_{\bar{G}_n} \geq \mathrm{KL}\left(\mathrm{Ber}(\mathbb{E}_{\theta}[Z]) || \mathrm{Ber}(\mathbb{E}_{\theta'}[Z])\right)
\end{align}
for any $\theta'\in \Real^d$ and a measurable random variable $Z\in(0,1)$. As $\mathrm{span}(\cS^{d-1}) = \mathbb{R}^{d}$, $\bar{G}_n$ will eventually be non-singular \citep{lattimore2017end}. Hence, let the eigendecomposition of $\bar{G}_n$ be $\sum_{i=1}^d \lambda_iu_iu_i^\top$ with $\lambda_i$'s arranged in descending order.
Let us choose $\theta'$ to be $\theta + \alpha u_d$, where $\alpha$ is a step size to be determined. This gives the L.H.S. of~\eqref{eq:info_ineq} as 
\begin{align}
\label{eq:main_2}
    \|\theta - \theta'\|^2_{\bar{G}_n} = \alpha^2\|u_d\|^2_{\bar{G}_n} = \alpha^2 \lambda_d~.
\end{align}
Let us define an $\epsilon$-neighbourhood about the optimal arm $\mathrm{OPT}(\theta)$ for a bandit parameter $\theta$ as 
\begin{align*}
    \mathrm{OPT}_\epsilon(\theta) \triangleq \{ x \in \mathcal{X} \mid x^\top \theta \geq \sup\nolimits_x x^\top \theta - \epsilon \}~.
\end{align*}
The step size $\alpha$ needs to be chosen such that $\theta$ and $\theta'$ are close in norm. At the same time, we need to ensure that the optimal arm for $\theta$ is sub-optimal for $\theta'$ and vice-versa. This motivates us to find an $\alpha$ such that 
\begin{align}
\label{eq:disjoint_condition}
    \mathrm{OPT}_\epsilon(\theta) \cap \mathrm{OPT}_\epsilon(\theta + \alpha u_d) = \emptyset~.
\end{align}
We denote the number of times an arm in the $\epsilon$-neighbourhood of $\theta$ is played as 
\begin{align*}
    \mathrm{N}_{\epsilon, n}(\theta) \triangleq \sum\nolimits_{s=1}^n\mathrm{1}\{A_s \in \mathrm{OPT}_\epsilon(\theta)\}
\end{align*}
Now, we define $Z$ in \eqref{eq:info_ineq} as the fraction of times in $n$ rounds that an arm in the $\epsilon$-neighbourhood of $\theta$ is played, i.e., $Z \triangleq \frac{\mathrm{N}_{\epsilon, n}(\theta)}{n}$. 
Then, as $\mathrm{OPT}_\epsilon(\theta)$ and $\mathrm{OPT}_\epsilon(\theta')$ are disjoint, $\mathbb{E}_\theta[Z]$ and $\mathbb{E}_{\theta'}[Z]$  will be different (in fact this is what a sub-linear regret algorithm is expected to do), and hence the right hand side of \eqref{eq:info_ineq} will be positive. To choose the step size $\alpha$, we exploit the geometry of the action-space. We refer to the Figure \ref{fig:main} from which it is clear that the amount $\alpha$ needed to perturb $\theta$ to $\theta'$ such that the disjoint condition \eqref{eq:disjoint_condition} holds would depend upon the orientation of $u_d$ with respect to $\mathrm{OPT}(\theta)$, which is $\theta$ itself, in our geometry. \footnote{For tackling the ellipsoidal case, we need to change variables (`whitening') to bridge to the sphere case; see Appendix~\ref{app:main}.} 

\tikzset{every picture/.style={line width=0.75pt}} 
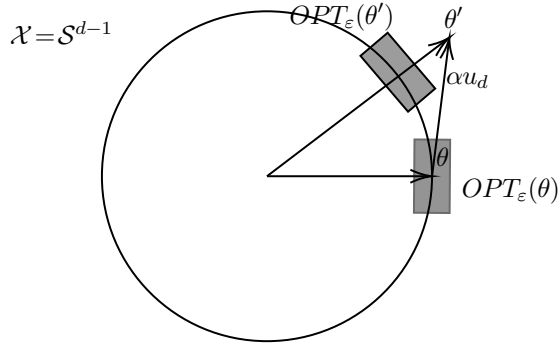
\begin{figure}
\centering
\tikzset{every picture/.style={line width=0.75pt}} 

\begin{tikzpicture}[x=0.75pt,y=0.75pt,yscale=-1,xscale=1]

\draw    (243,168) -- (324.25,168) ;
\draw [shift={(326.25,168)}, rotate = 180] [color={rgb, 255:red, 0; green, 0; blue, 0 }  ][line width=0.75]    (10.93,-3.29) .. controls (6.95,-1.4) and (3.31,-0.3) .. (0,0) .. controls (3.31,0.3) and (6.95,1.4) .. (10.93,3.29)   ;
\draw    (326.25,168) -- (334.75,99.98) ;
\draw [shift={(335,98)}, rotate = 97.13] [color={rgb, 255:red, 0; green, 0; blue, 0 }  ][line width=0.75]    (10.93,-3.29) .. controls (6.95,-1.4) and (3.31,-0.3) .. (0,0) .. controls (3.31,0.3) and (6.95,1.4) .. (10.93,3.29)   ;
\draw    (243,168) -- (333.41,99.21) ;
\draw [shift={(335,98)}, rotate = 142.73] [color={rgb, 255:red, 0; green, 0; blue, 0 }  ][line width=0.75]    (10.93,-3.29) .. controls (6.95,-1.4) and (3.31,-0.3) .. (0,0) .. controls (3.31,0.3) and (6.95,1.4) .. (10.93,3.29)   ;
\draw   (159.75,168) .. controls (159.75,122.02) and (197.02,84.75) .. (243,84.75) .. controls (288.98,84.75) and (326.25,122.02) .. (326.25,168) .. controls (326.25,213.98) and (288.98,251.25) .. (243,251.25) .. controls (197.02,251.25) and (159.75,213.98) .. (159.75,168) -- cycle ;
\draw  [fill={rgb, 255:red, 0; green, 0; blue, 0 }  ,fill opacity=0.39 ] (290.15,107.25) -- (303.85,95.58) -- (327.85,123.75) -- (314.15,135.42) -- cycle ;
\draw  [color={rgb, 255:red, 0; green, 0; blue, 0 }  ,draw opacity=0.58 ][fill={rgb, 255:red, 0; green, 0; blue, 0 }  ,fill opacity=0.42 ] (317.39,149.43) -- (335.39,149.57) -- (335.11,186.57) -- (317.11,186.43) -- cycle ;

\draw (253,79) node [anchor=north west][inner sep=0.75pt]   [align=left] {$\displaystyle OPT_{\epsilon }( \theta ')$};
\draw (331,81) node [anchor=north west][inner sep=0.75pt]   [align=left] {$\displaystyle \theta '$};
\draw (327,152) node [anchor=north west][inner sep=0.75pt]   [align=left] {$\displaystyle \theta $};
\draw (332,114) node [anchor=north west][inner sep=0.75pt]   [align=left] {$\displaystyle \alpha u_{d}$};
\draw (340,167) node [anchor=north west][inner sep=0.75pt]   [align=left] {$\displaystyle OPT_{\epsilon }( \theta )$};
\draw (112,88) node [anchor=north west][inner sep=0.75pt]   [align=left] {$\displaystyle \mathcal{X} =\mathcal{S}{^{d-1}}$};
\end{tikzpicture}
\caption{\footnotesize{The arm set $\cS^{d-1}$ with the bandit parameter $\theta = v$, the optimal arm. The direction $u_d$ is approximately orthogonal to $v$. The task is to find $\alpha$, the amount of perturbation of $\theta$ to $\theta'$ in the direction of $u_d$, so that the $\epsilon$ neighbourhood of both $\theta$ and $\theta'$ is disjoint}}
\label{fig:main}
\end{figure}

We use the Davis-Kahan matrix perturbation theorem (see Lemma~\ref{lemma:Davis_Kahan}) to show $\mathrm{OPT}(\theta)$ (denoted here on as $v$ for notational simplicity) and $u_d$ are approximately orthogonal. In its simplest form, the Theorem states that given two matrices $A$ and $H$ (symmetric), if $\lambda_1(A)$ and all of $\lambda_2(A+H),\cdots,\lambda_d(A+H)$ are separated by $\delta$, then the eigenvector $v$ corresponding to $\lambda_1(A)$ and the eigenvectors $u_2,\cdots,u_d$ corresponding to $\lambda_2(A+H),\cdots,\lambda_d(A+H)$, make an inner product of at most $\norm{H}/\delta$, i.e., $|\max_{i=2,\cdots,d}v^\top u_i| \leq \norm{H}/\delta$.

We use $A$ as $nvv{^\top}$ and $H$ as $\bar{G}_n - nvv^\top$ such that $v$ is eigen-vector of $nvv{^\top}$ corresponding to $\lambda_1(A) = n$ and $u_2,\cdots,u_d$ are the eigenvectors of $A+H = \bar{G}_n$ corresponding to the eigenvalues $(\lambda_2(\bar{G}_n), \cdots, \lambda_d(\bar{G}_n))$. To find an upper bound of $\langle v, u_d\rangle$, we get an upper bound on the spectral norm of $\bar{G}_n - nvv^\top$ and the eigen-gap between $\lambda_1(nvv^\top) = n$ and $\lambda_2(\bar{G}_n), \cdots, \lambda_d(\bar{G}_n)$.

We know from Weyl's Lemma (see Lemma~\ref{lemma:Weyl's}) that $\lambda_i(\bar{G}_n) \leq \lambda_i(nvv^\top) + \norm{\bar{G}_n - nvv^\top} = \norm{\bar{G}_n - nvv^\top}$, for $i \in \{2,\cdots,d\}$ as $\lambda_i(nvv^\top) = 0$ for all $i \in \{2,\cdots,d\}$.  
Thus it suffices to upper bound the spectral norm, $\norm{\bar{G}_n - nvv^\top}$. To do so we decompose it into two cases: when the action played ($A_s$) belongs to the optimal set $\mathrm{OPT}_\epsilon(\theta)$, and when it does not. This yields the following
\begin{align} 
\label{eq:step_1}
    \|\bar{G}_n - n vv^\top\| 
    &\leq \sup_{A_s \in \mathrm{OPT}_\epsilon(\theta) }2\norm{A_s - v}\expect{\mathrm{N}_{\epsilon, n}(\theta)}\nonumber\\ &\quad\quad+ 4\expect{n - \mathrm{N}_{\epsilon, n}(\theta)},
\end{align}
where we have used the fact that arms in $\cS^{d-1}$ has maximum norm of $1$. We use the geometry of the action-space to determine that for any $A_s \in \mathrm{OPT}_\epsilon(\theta)$, we have $\norm{A_s - v} \leq O(\sqrt{\epsilon})$ (see Appendix \ref{app:main} for details). Now,
we use the regret property of algorithm to show $\mathbb{E}_\theta[\mathrm{N}_{\epsilon, n}(\theta)] \geq n - c\sqrt{n}/\epsilon$, as
\begin{align}
\label{eq:number_of_optimal_arms}
      c\sqrt{n} &\geq\mathbb{E}\left[\sum_{s=1}^n \max_x x^\top\theta - A_s^\top\theta\right] \nonumber\\
     &\geq \mathbb{E}\left[\sum_{s=1}^n \mathrm{1}\{A_s \notin \mathrm{OPT}_\epsilon(\theta)\}(\max_x x^T\theta - A_s^T\theta)\right]\nonumber\\
     &\geq \eps(\mathbb{E}\left[n-\mathrm{N}_{\epsilon, n}(\theta)\right])\;.  
\end{align}
Thus, we have an upper bound on the spectral norm of $\bar{G}_n - n vv^\top$ as 
    $\|\bar{G}_n - n vv^\top\| \leq O\big(\sqrt{\epsilon}n + \frac{\sqrt{n}}{\epsilon}\big).$
Now choosing $\epsilon$ to be of the order $\frac{1}{\sqrt{n}}$ with a sufficiently large constant factor, we ensure the norm of $\|\bar{G}_n - n vv^\top\| \leq 0.01n$ for all $n$ more than some finite $n_0$. Thus the eigen-gap is at least $0.99n$ and $|\langle v, u_d\rangle| \leq 1 /99 \approx 0$. With this orientation of $u_d$, and the choice of $\epsilon$, we prove that $\alpha$ has to be of the order of $\frac{1}{n^{1/4}}$ for $\mathrm{OPT}_\epsilon(\theta) \cap \mathrm{OPT}_\epsilon(\theta + O(1/n^{1/4}) u_d) = \emptyset$. The details of this result are provided in the Appendix \ref{app:main}. Now, using our estimates of $\mathbb{E}_\theta[N_{\epsilon,n}(\theta)]$ (see equation \eqref{eq:number_of_optimal_arms}), and by our choice of $\epsilon = O(1/\sqrt{n})$ we have $\mathbb{E}_\theta[Z]$ close to $1$ and $\mathbb{E}_{\theta'}[Z]$ close to $0$. (See the previous discussion for our motivation to choose disjoint $\epsilon$ neighbourhood sets). This gives the right hand side of equation \eqref{eq:info_ineq} a constant $c$ (See appendix \ref{app:main} for full details). Therefore combining with equation \eqref{eq:main_2} we have $\lambda_d \geq c/\alpha^2 = \Omega(\sqrt{n})$.
\end{proof}

\begin{remark}
We observe in the above proof that the eigenvector corresponding to the minimum eigenvalue of $\bar G_n$ is approximately orthogonal to the direction of optimal arm. In fact, we show an even stronger fact: {\em every} eigenvector corresponding to each of the eigenvalues, starting from the second largest to the minimum, must lie in an approximately orthogonal space of the optimal direction.  
\end{remark}
\section{MORE GENERAL ACTION SPACES}
\subsection{Hyper-Surfaces With Continuous Hessian}
The basic conclusion of Theorem \ref{thm:main-result} (growth of the minimum eigenvalue of the design matrix) can be extended beyond LCH action spaces to hyper-surfaces of the form $\{x : f(x) = c\}$ where $f$ is just $C^2$ and locally convex, without requiring a constant Hessian in a neighborhood. However, this comes potentially at the cost of the $\sqrt{n}$ growth rate of the minimum eigenvalue. 

\begin{definition}[Locally Convex surface]
\label{def:locally_convex_action}
Consider an action space $\mathcal{X} = \{x \in \mathbb{R}^d : f(x) = c\}$, where $f : \mathbb{R}^d \to \mathbb{R}$ is a $C^2$ function (i.e., all second order partial derivatives exist and are continuous). With $\mathrm{OPT}_{\mathcal{X}}(\theta)$ being the optimal arm defined as before, let the Hessian of $f$ at $\mathrm{OPT}_{\mathcal{X}}(\theta)$, denoted as $\nabla^2f(\mathrm{OPT}_{\mathcal{X}}(\theta))$, be positive definite. Then, $\mathcal{X}$
is said to be a Locally Convex surface.
\end{definition}
\begin{remark}
Locally Convex action spaces are more general than LCH action spaces because they satisfy the property of the Hessian being positive definite and continuous at the optimal arm, and do away with the additional requirement of being constant in a neighbourhood of the optimal arm. This definition can also be extended to cover the situation when $f$ is a convex function and the action space is defined as a sub-level set $\{x:f(x) \leq c\}$ as described in the next subsection.
\end{remark}

For such action-spaces we show that the minimum eigenvalue still enjoys a polynomial growth rate albeit potentially at a rate less than $\sqrt{n}$. 
\begin{theorem}
\label{thm:locally_convex}
Let $\cX$ be a Locally Convex action space and $\Bar{G}_n$, be  the expected design matrix. For any bandit algorithm which suffers expected regret at most $O(\sqrt{n})$, there exists a real number $s$ in the half-open interval $(0,\frac{1}{2}]$ such that 
\begin{equation*}
    \lambda_{\min}(\Bar{G}_n) = \Omega(n^s)~.
\end{equation*}
\end{theorem}
The idea of the proof is to approximate a local neighbourhood around the optimal arm by an LCH surface, and to argue sufficient separation of optimal neighbourhoods in the LCH surface to ensure that the optimal neighbourhoods in the original neighbourhood are separated as well. The rest of the argument remains the same as for LCH spaces, and can be found in full detail in Appendix \ref{sec:app-convex}.
\begin{remark}
The exponent $s$ defined in Theorem \ref{thm:locally_convex} depends explicitly on the geometry of the action-space. In general it depends on how well the surface approximates a LCH surface and in the Appendix \ref{sec:app-convex} we provide a full methodology on how to calculate $s$. Here we just remark that for LCH action surfaces we recover the $\sqrt{n}$ growth rate.
\end{remark}

\subsection{Action Spaces With ``Volume"}
Our results of Theorem \ref{thm:main-result} and Theorem \ref{thm:locally_convex} continue to hold even if we replace the action space $\{x: f(x) = c\}$ by action spaces which are defined as sub-level sets for a $C^2$ convex function function $\{x: f(x) \leq c\}$. Note that any bandit algorithm would make the greedy choice with respect to an estimate of the bandit parameter $\theta$. This effectively reduces the playable action space to that on the surface (maximization of a linear function over a convex domain occurs at the boundary) and we can take the action space to be the surface. We can also prove this using the basic principles used in the proof of Theorem \ref{thm:main-result} and by extension Theorem \ref{thm:locally_convex}. For example, let $\mathcal{X}$ be a ball in $\mathbb{R}^d$, and $\theta$ be a bandit parameter. Without loss in generality, we can take $\| \theta \| = 1$ (otherwise, we can make a change of variable argument like in the ellipsoidal case to show that the result holds, now including a factor of $\| \theta \|$). The crux of the proof relies on two factors. (a) First, to show that $u_d$, the eigen-vector, corresponding to the minimum eigen-value is approximately perpendicular to $\mathrm{OPT}_{\mathcal{X}}(\theta)$. (b) Second to show that $\mathrm{OPT}_{\epsilon, \mathcal{X}}(\theta) \cap \mathrm{OPT}_{\epsilon, \mathcal{X}}(\theta + \alpha u_d) = \emptyset$. For the first part, the main idea is to bound the norm $\|A_s - \mathrm{OPT}_{\mathcal{X}}(\theta)\|$ for any $A_s \in \mathrm{OPT}_{\epsilon, \mathcal{X}}(\theta)$. To do so, note that we had used $\|A_s\| = 1$ for the surface action spaces. Now, when action space is a ball, we can take an upper bound on
$\|A_s\| \leq 1$ while the remainder of the proof follows as before. For the second part, note that the boundary of $\mathrm{OPT}_{\epsilon, \mathcal{X}}(\theta)$ for a ball is simply $\mathrm{OPT}_{\epsilon, \mathcal{S}^{d-1}}(\theta)$ as defined earlier. Thus ensuring the $\epsilon$-optimal neighbourhoods of the sphere are separated ensures the separation of the $\epsilon$-optimal neighbourhoods of the ball. Now we can proceed as before and find the order of the separation to get the same conclusions as in Theorem \ref{thm:main-result} and Theorem \ref{thm:locally_convex}.

%% file: experiments.tex
\section{NUMERICAL RESULTS}
\label{sec:experiments}
\begin{figure}[htbp]
\centering
    \subfloat[$d=3$]{
    \includegraphics[width=0.5\linewidth]{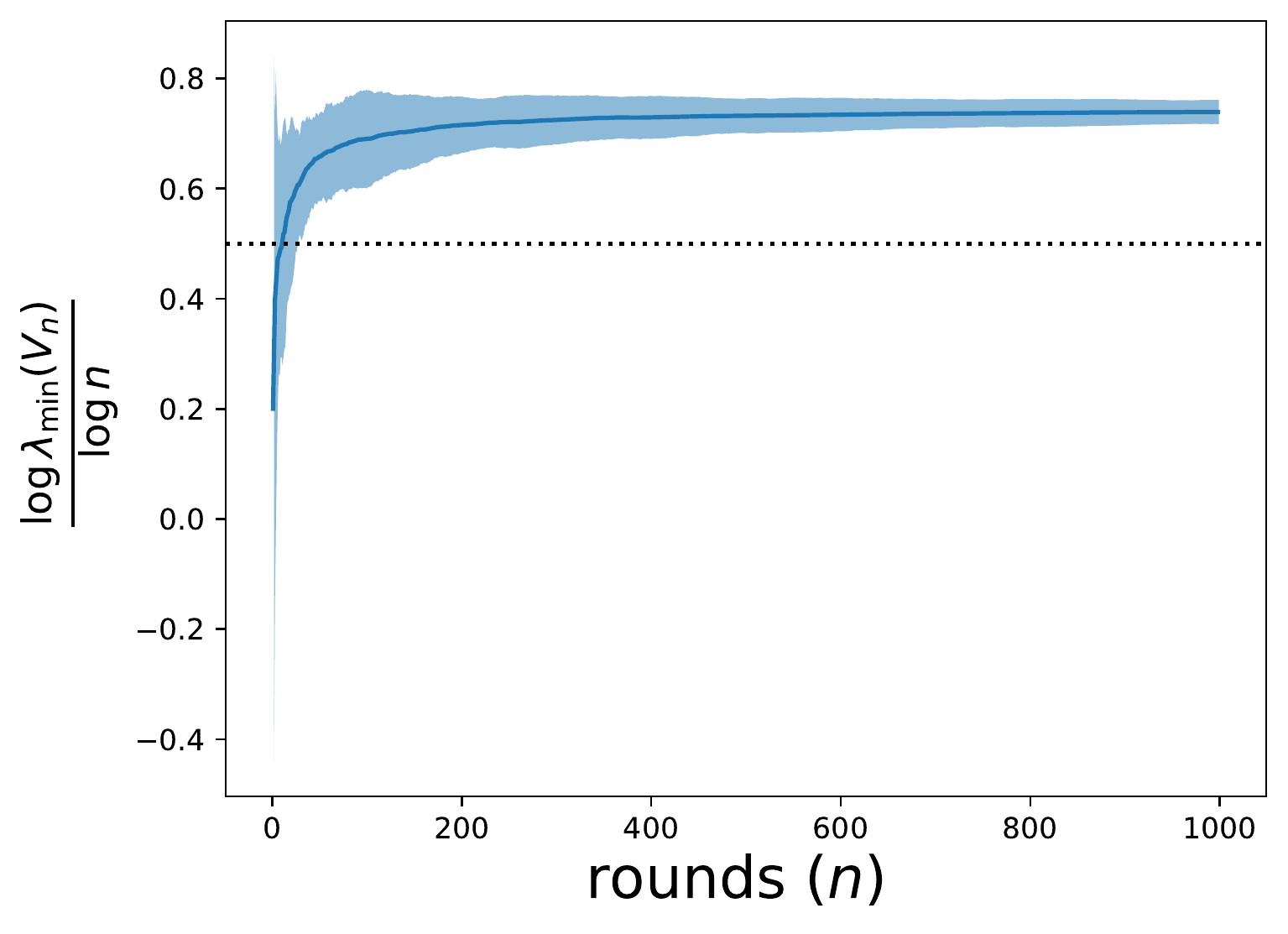}}
    \subfloat[$d=5$]{
    \includegraphics[width=0.5\linewidth]{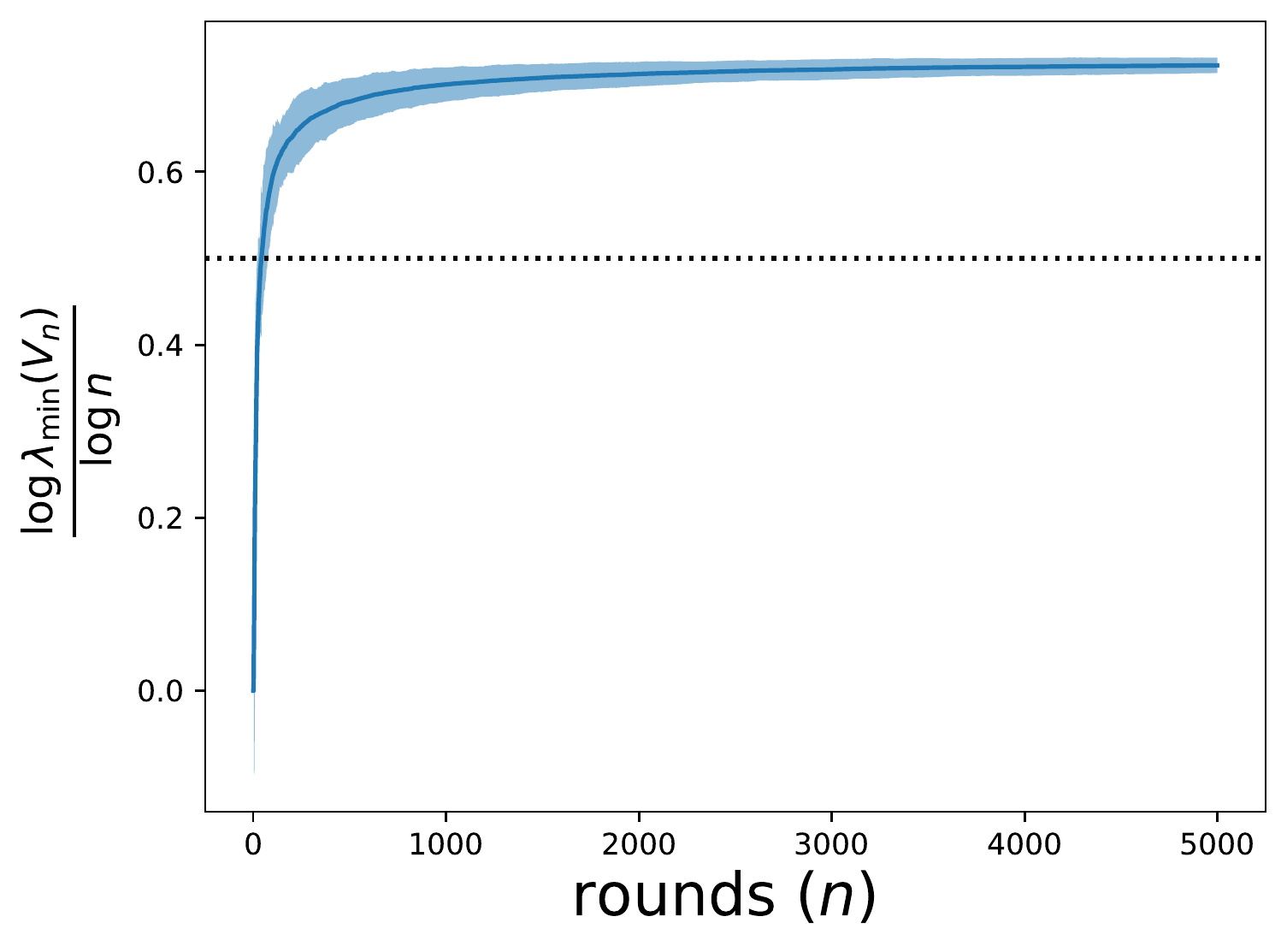}}\\
    \subfloat[$d=10$]{
    \includegraphics[width=0.5\linewidth]{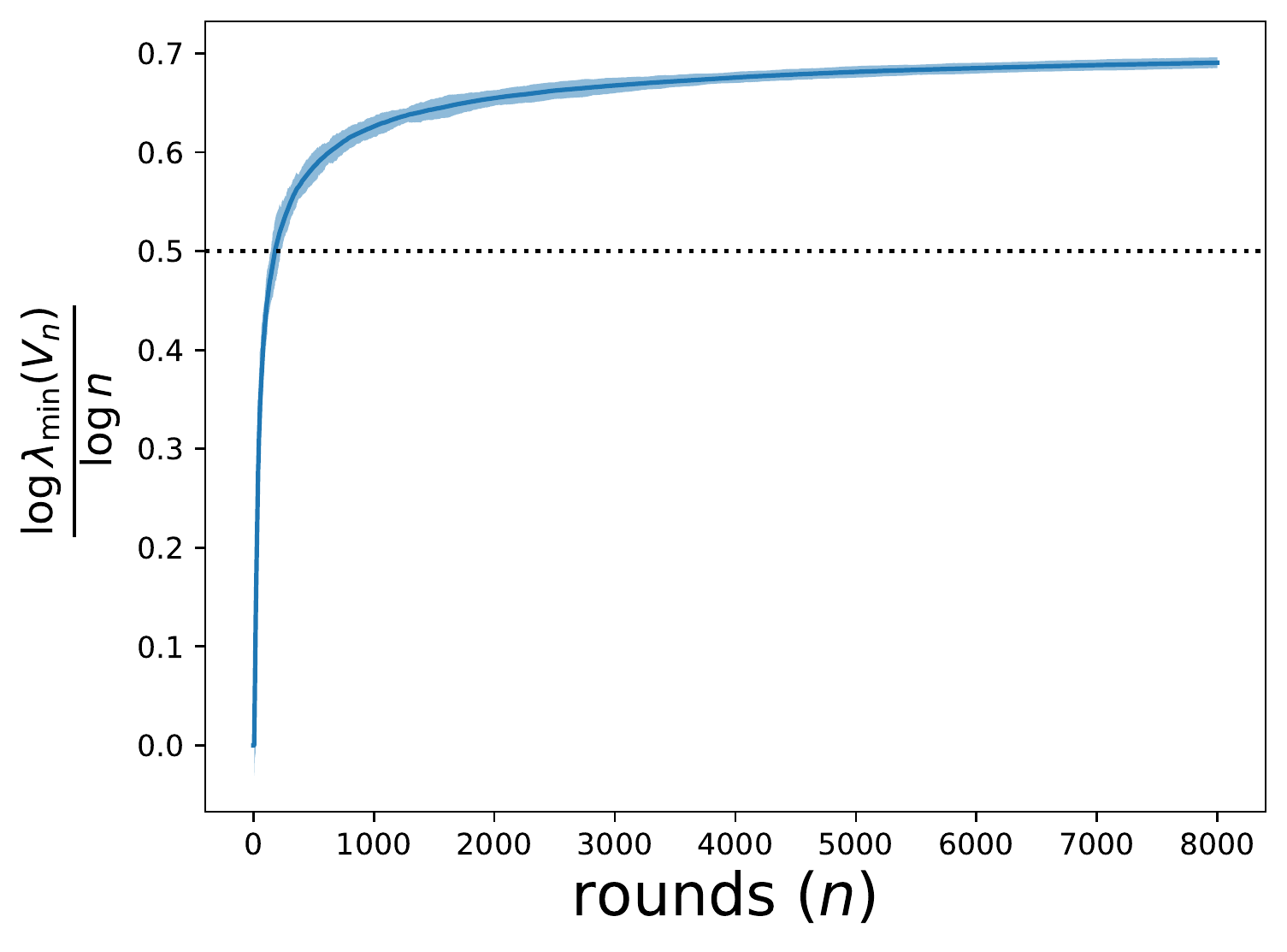}} 
    \subfloat[$d=20$]{
    \includegraphics[width=0.5\linewidth]{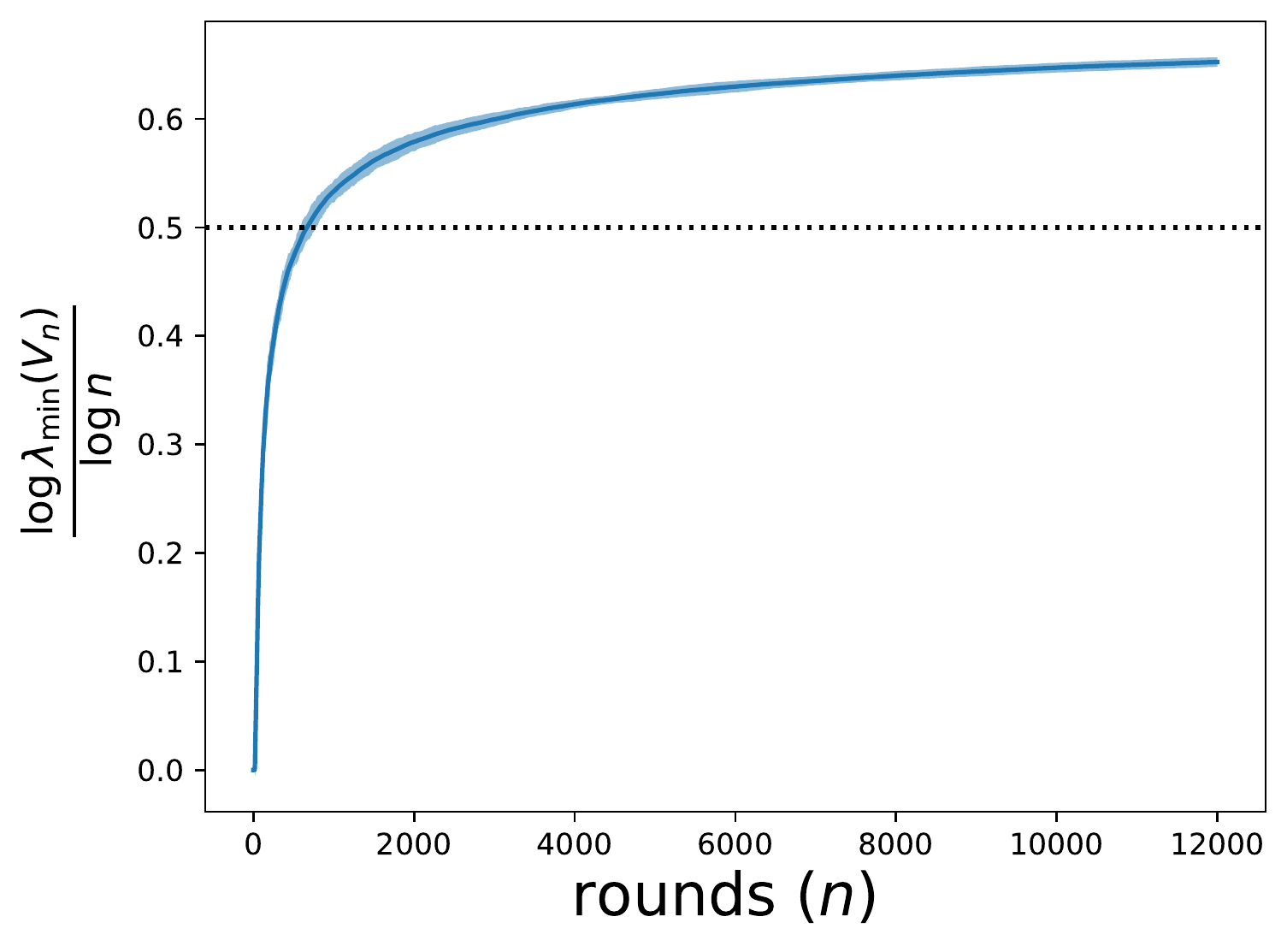}}
    \caption{\footnotesize{Scaling of the minimum eigenvalue of design matrix generated by the Thompson sampling algorithm with time. The plots represent averages over $20$ independent trials. The $X$ axis denotes the number of rounds $n$ and the $Y$-axis denotes $\frac{\log{\lambda_{\min}(V_n)}}{\log{n}}$. The dotted black line is the constant (exponent) $1/2$. Note how $\frac{\log{\lambda_{\min}(V_n)}}{\log{n}}$ crosses and settles to a value above $1/2$ in each case.}}
    \label{fig:exp_1}
    
\end{figure}

In this section, we carry out experiments to understand the rate of growth of the minimum eigenvalue. We find, using a "good" bandit algorithm and a "good" action space, that the minimum eigen-value of the design matrix grows at a rate of more than $\sqrt{n}$, with high probability. We use the well-known linear Thompson Sampling (TS) algorithm  \citep{agrawal2013thompson} as a representative algorithm for linear bandits. 
It is well known that the regret of TS is $\Tilde{O}(\sqrt{n})$ with high probability, where the $\Tilde{O}$, hides logarithmic factors. We use a spherical action space $\cX = \cS^{d-1}$ as a candidate for the LCH action surfaces and we fix the unknown bandit parameter $\theta$ at $e_1$. We use a regularization parameter of $\lambda = 1.0$.   

We observe from Figure \ref{fig:exp_1} the dispersion and mean trend of the sample trajectory dependent quantity $\log{\lambda_{\min}(V_n)}/\log{n}$. We expect to see to see this term cross the benchmark line of $1/2$ for all $n$ more than a finite $n_0$. In order to demonstrate the high probability phenomenon, we form a high confidence band of the mean observation of $\log{\lambda_{\min}(V_n)}/\log{n}$ with three standard deviations of width. We plot this confidence band and we show that the lower envelope of the band remains more than $1/2$. It could be observed that the trend of $\log{\lambda_{\min}(V_n)}/\log{n}$ remains increasing, but this could be accounted for the hidden logarithmic factors in the regret of the TS algorithm itself.

\begin{figure}[htbp]
\centering
    \subfloat[$n_0$ trend with dimension $d$]{
    \includegraphics[height=0.4\linewidth,width=0.5\linewidth]{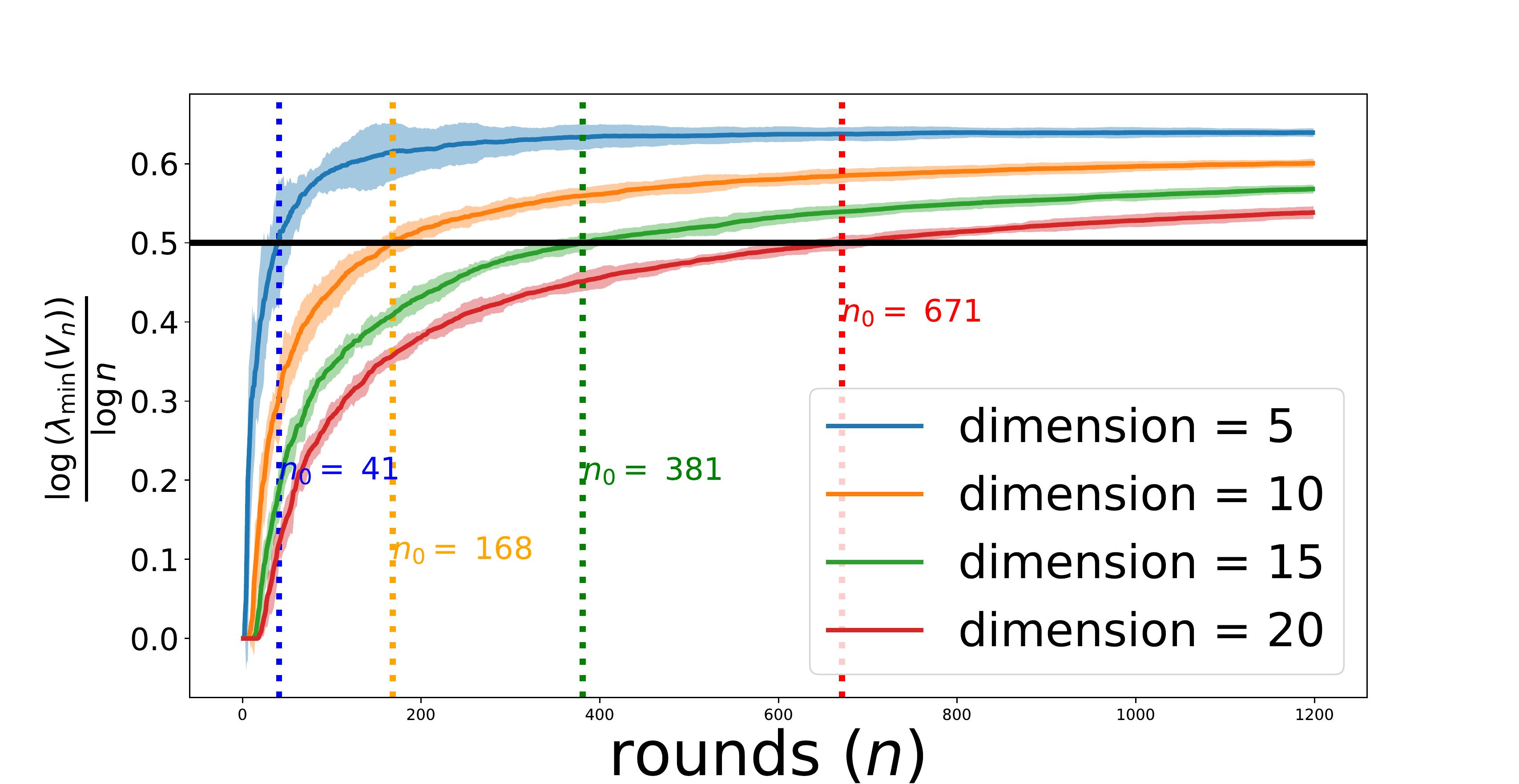}}
    \subfloat[$\gamma$ trend with dimension $d$]{
    \includegraphics[height=0.4\linewidth,width=0.5\linewidth]{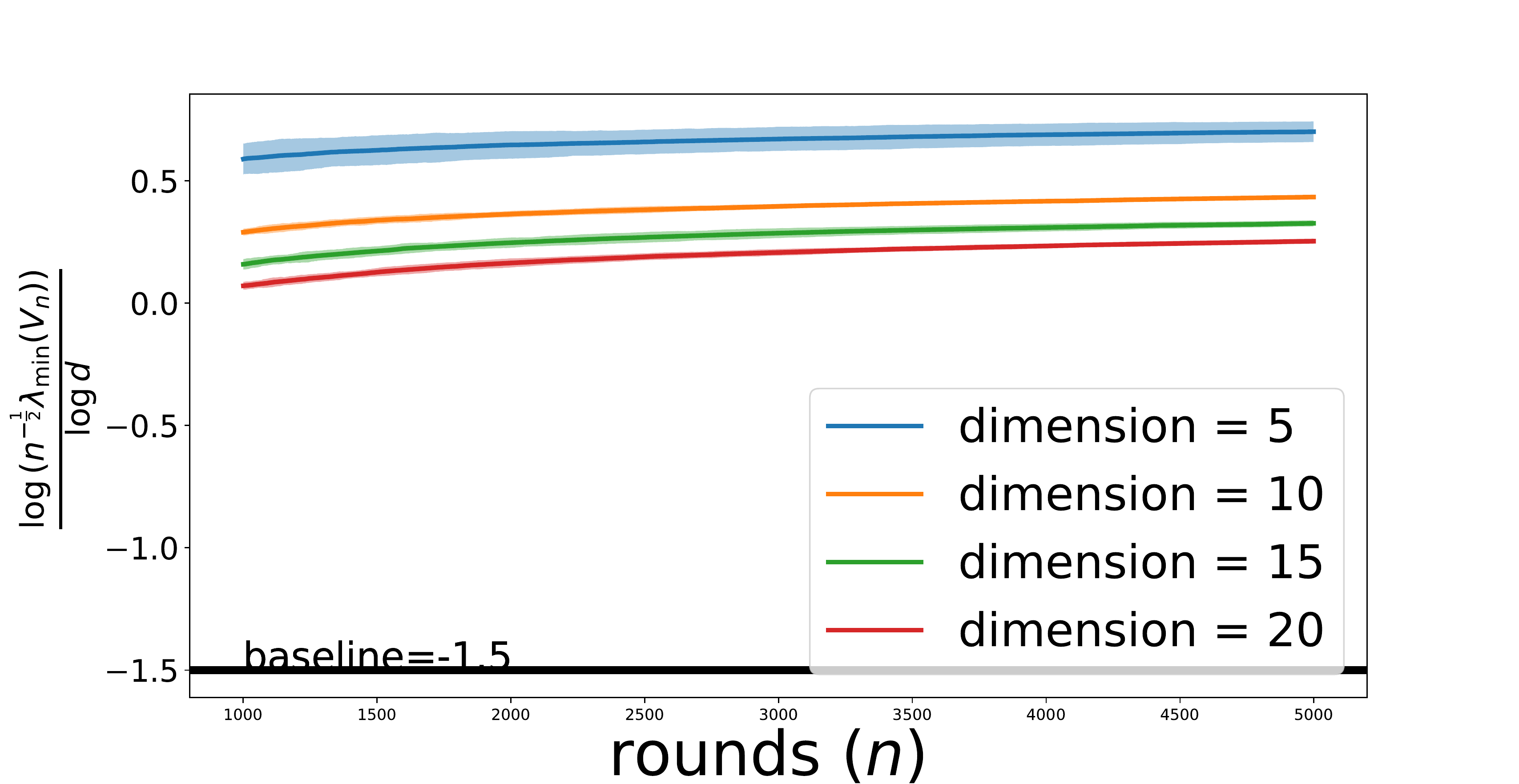}}
    \caption{\footnotesize{Scaling of the theoretical constants $\gamma$ and $n_0$ with the dimension $d$. Note that $n_0$ increases with the dimension $d$ whereas $\gamma$ decreases with the dimension $d$. Since the algorithmic constants varies appropriately with the dimension $d$, in the case of TS, our experimental results corroborates with our theory.}}
    \label{fig:exp_2}
    
\end{figure}
We note that for the spherical action space, the constants $n_0$ and $\gamma$ of Theorem~\ref{thm:main-result} would vary with the dimension $d$ as $\Omega(d^3)$ and $\Omega(1/d^{3/2})$ for TS (see proof of Theorem~\ref{thm:main-result}). In Figure \ref{fig:exp_2} we explicitly calculate the the mean $n_0$, as the time above and over which $\log{\lambda_{\min}(V_n)}/\log{n}$ crosses the $0.5$ bench mark, and observe that our theoretical lower bound of $n_0 = \Omega(d^3)$ is a loose lower bound and that in practical settings the $n_0$ can be much less than the order $d^3$. Similarly for $\gamma$, defined as $\liminf{\frac{\lambda_{\min}(V_n)}{\sqrt{n}}}$, is plotted as the mean trend of $\frac{\log{\lambda_{\min}(V_n)/\sqrt{n}}}{\log{d}}$ for rounds $n$ more than $1000$, as a reasonable upper estimate of $n_0$ and see that this quantity remains well above the benchmark of $-1.5$, thus showing that $\gamma = \Omega(1/d^{3/2})$, is indeed a loose lower bound.


%% file: model-selection.tex
\section{APPLICATIONS}
\label{sec:applications}
We now give two instances where the eigenvalue bound obtained in Theorem~\ref{thm:main-result} can be used to its advantage due to its implication of fast inference about the unknown parameter. We first start by noting that for parameter estimation we need a high probability version of the Theorem~\ref{thm:main-result}.

In Theorem~\ref{thm:main-result}, the eigenvalue bound is shown on the expected design matrix. In the simulations (Section~\ref{sec:experiments}), we observe (over multiple trials) that the lower envelope of $\lambda_{\min}(V_n)$ scales as $\Omega (\sqrt{n})$, which hints towards a high probability result. It turns out that under an appropriately defined \emph{stability} condition, the expectation bound can be translated to a high probability bound. In Appendix \ref{sec:high_probability} we formally define this stability condition and its consequences. However, in this section, for clarity of exposition, we take the high probability bound as granted, and using this, we consider a couple of applications in bandit model selection and clustering. Recall that we define the Gram matrix as $V_n=\sum_{s=1}^{n}A_s A_s^\top$.
\begin{assumption}
\label{ass:high-prob}
Given an action-set $\cX$ which belongs to the class of \emph{Locally Constant Hessian} surfaces (see Definition~\ref{def:action}) for a bandit parameter $\theta$, for any bandit algorithm which suffers regret, $R_n(\theta)$, at most $O(\sqrt{n})$, there exists a $\delta \in (0,1]$, such that
\begin{align*}
 \lambda_{\min}(V_n) = \Omega(\sqrt{n})~,   
\end{align*}
with probability at least $1-\delta$.
\end{assumption}
\begin{remark}
Using the aforementioned \emph{stability} condition we can get similar high probability versions for \emph{Locally Convex} action spaces. However as discussed earlier this could potentially loose the $\sqrt{n}$ growth rate of the minimum eigen-value. For ease of exposition, we shall assume in this section that the action-space in consideration is an LCH surface unless mentioned otherwise.
\end{remark}
\subsection{Model Selection in Linear Bandits}
A natural measure of complexity of linear bandits is the norm of the parameter $\theta$. Typically it is assumed that $\theta$ lies in a norm ball with known radius, i.e., $\|\theta\| \leq b$ \citep{oful,pmlr-v108-chatterji20b}. This leads to non-adaptive algorithms and the algorithms use $b$ as a proxy for the problem complexity, which can be a huge over-estimate. We analyze a linear bandit algorithm, that adapts to the problem complexity $\|\theta\|$, and as a result, the regret obtained will depend on $\|\theta\|$. Specifically, we start with an over estimate of $\|\theta\|$, and successively refine this estimate over multiple epochs. We show that this refinement strategy yields a consistent sequence of estimates of $\|\theta\|$, and as a consequence, our regret bound depends on $\|\theta\|$, but not on its upper bound $b$. \citet{ghosh2021problem} considers a similar model selection problem in a related setting of \emph{stochastic contextual bandits}, with restrictive assumptions on the contexts. We show that model selection guarantees continue to hold even in the context free setting without any additional assumption by using the result in Theorem~\ref{thm:main-result} and the subsequent Assumption~\ref{ass:high-prob}. Before that, we restate the algorithm of \citet{ghosh2021problem}.


\subsubsection{Algorithm and its Regret Bound}

The algorithm -- called Adaptive Linear Bandits (ALB) -- uses the \texttt{OFUL} algorithm of \citet{oful} as a black-box. The learning proceeds in epochs -- at each epoch $i \geq 1$, \texttt{OFUL} is run for $n_i=2^{i-1}n_1$ episodes with confidence level $\delta_i=\frac{\delta}{2^{i-1}}$ and norm estimate $b^{(i)}$, where the initial epoch length $n_1$ and confidence level $\delta$ are parameters of the algorithm. We begin with $b$ as an initial (over) estimate of $\norm{\theta}$, and at the end of $i$-th epoch, based on the confidence set $\cB_{n_i}$ build by \texttt{OFUL}, we choose the new estimate as $ b^{(i+1)} = \max_{\theta \in \cB_{n_{i}}} \|\theta\|$. We argue that this sequence of estimates is indeed consistent, and as a result, the regret depends on $\norm{\theta}$.

\begin{algorithm}[t!]
  \caption{Adaptive Linear Bandit (ALB)}
  \begin{algorithmic}[1]
 \STATE  \textbf{Input:} An upper bound $b$ of $\norm{\theta}$, initial epoch length $n_1$, confidence level $\delta \in (0,1]$
 \STATE Initialize estimate of $\|\theta\|$ as $b^{(1)} = b$, set $\delta_1=\delta$
  \FOR{ epochs $i=1,2 \ldots $}
  \STATE Play \texttt{OFUL} with norm estimate $b^{(i)}$ for $n_i$ rounds with confidence level $\delta_i$ 
  \STATE Refine estimate of $\|\theta\|$ as $ b^{(i+1)}\! =\! \max_{\theta \in \cB_{n_{i}}} \!\!\|\theta\|$
  \STATE Set $n_{i+1} = 2 n_{i}$, $\delta_{i+1} = \frac{\delta_i}{2}$
    \ENDFOR
  \end{algorithmic}
  \label{algo:norm}
\end{algorithm}
Now, we present the main result of this section. We show that, by virtue of Theorem~\ref{thm:main-result} and the subsequent Assumption~\ref{ass:high-prob}, the norm estimates $b^{(i)}$ computed 
by \texttt{ALB} (Algorithm~\ref{algo:norm}) indeed converges to the true norm $\norm{\theta}$ at an exponential rate with high probability.

\begin{lemma}[Convergence of norm estimates]
\label{lem:sequence}
Suppose Assumption~\ref{ass:high-prob} holds. Also, suppose that, for any $\delta \in (0,1]$, the length $n_1$ of the initial epoch satisfies 
\begin{align*}
        n_1 \!\geq\! \max\! \left\{n_0, C \,d^2 \left ( \max\{p,q\}  \, b^{(1)} \right)^4 \right\},
    \end{align*}
where $p \!=\! O\!\left(\!\frac{1}{\sqrt{\gamma_0}}\!\right)$, $q \!=\! O\!\left(\!\sqrt{\!\frac{\log (n_1/\delta)}{\gamma_0}}  \right)$, and $C > 1$ is some sufficiently large universal constant. Then, with probability exceeding $1-4\delta$, the sequence $\{b^{(i)}\}_{i=1}^\infty$ converges to $\|\theta\|$ at a rate $\mathcal{O}\left(\frac{i}{2^{i/4}}\right)$, and $b^{(i)} \leq c_1 \|\theta\| + c_2$, where $c_1,c_2 > 0$ are universal constants.
\end{lemma}

\textbf{Comparison with prior work.} \citet{ghosh2021problem} consider the setting stochastic contextual linear bandits, and show that the norm estimates converge at a rate $\mathcal{O}\left(\frac{i}{2^{i/2}}\right)$. This seemingly better rate, however, comes at the cost of restrictive assumptions on the contexts. Specifically, they assume that the minimum eigenvalue of the conditional covariance matrix (given observations up to time $t$) is bounded away from zero. This restriction on the contexts severely limits the applicability of their result. In contrast, we obtain a slightly worse, but still exponential, convergence rate $\mathcal{O}\left(\frac{i}{2^{i/4}}\right)$, albeit with a milder assumption in the context of our main result (Theorem \ref{thm:main-result}).

Armed with the above result, we can now prove a sublinear regret bound of \texttt{ALB}.
\begin{corollary}[Cumulative regret of \texttt{ALB}]
\label{thm:norm}
Fix any $\delta \in (0,1]$, and suppose that the hypothesis of Lemma \ref{lem:sequence} holds. Then, with probability exceeding $1- 6\delta$, \texttt{ALB} enjoys the regret bound
\begin{align*}
    R(n) &\!  =\!  \Tilde{\mathcal{O}}\!\left(\| \theta\| \sqrt{dn\log n_1}+d\sqrt{n\log n_1\log(n_1/\delta)} \!\right)\\ &\leq \Tilde{\mathcal{O}} \left( (\|\theta\|+1) \, d \sqrt{n \log (n/\delta)} \right),
\end{align*}
where $\Tilde{\mathcal{O}}$ hides a $\mathsf{polylog}(n/n_1)$ factor.
\end{corollary}
\begin{remark}
Note that the above regret depends on $\|\theta\|$ and hence is adaptive to the norm complexity. Furthermore, the term $(\|\theta\|+1)$ can be reduced to $(\|\theta\|+ \epsilon)$ for an arbitrary $\epsilon>0$ by increasing the length of initial epoch in ALB.
\end{remark}

Lemma~\ref{lem:sequence} and Corollary~\ref{thm:norm} together highlight the advantage of having an eigenvalue bound as in Theorem~\ref{thm:main-result} in this specific setting of model selection due to its implication in fast inference and regret minimization, simultaneously.

%% file: clustering.tex
\subsection{Clustering without Exploration in Linear Bandits}
\label{sec:cluster}
We consider a multi-agent system with $N$ users, which communicate with a ``center''. Moreover, the agents are partitioned into $k$ clusters. We aim to identify the cluster identity of each user so that collaborative learning is possible within a cluster. Note that previously \citep{clustering_online,ghosh2021collaborative} tackled the problem of online clustering in a stochastic contextual bandit framework, where the players have additional context information to make a decision. More importantly, in the mentioned works, several restrictive assumptions were made on the behavior of the stochastic context, such as, based on the observation upto time $t$, (a) the conditional mean is $0$, (b) the conditional covariance matrix is positive definite, time uniform with the minimum eigenvalue bounded away from $0$ and (c) the conditional variance of the contexts projected in a fixed direction is bounded. Apart from these three, there are a few additional technical assumptions made, see \cite[Lemma 1]{clustering_online} for example. We believe it is difficult to find natural examples of contexts where all these assumptions are satisfied simultaneously, and the authors of the aforementioned papers also do not provide any.

To this end, we provide clustering guarantee without these restrictive assumptions. We only require the action space satisfying Definition~\ref{def:action}, which includes standard spaces like spheres and ellipsoids. Furthermore, we obtain the underlying clustering without \emph{pure (forced) exploration}---thanks to the lower bound of Theorem~\ref{thm:main-result} and the subsequent Assumption~\ref{ass:high-prob}. This is quite useful in recommendation systems and advertisement placement, where forced exploration is often very tricky and not at all desired.
We stick to the framework of standard linear bandits, and we assume a clustering framework, where the user parameters $\{\theta^*_1,\ldots,\theta^*_N\}$ are partitioned into $k$ groups. All users in a cluster have the same preference parameter, and hence, without loss of generality, for all $j \in [k]$, we denote $\theta^*_j$ as the preference parameter for cluster $j$.  We employ the standard OFUL of \citet{oful} as our learning algorithm. Note that, by virtue of Theorem~\ref{thm:main-result} and the subsequent Assumption~\ref{ass:high-prob}, we can estimate the underlying parameter for any agent $i \in [N]$ in $\ell_2$ norm. We use this information to do the clustering task. Our goal here is to propose an algorithm that finds the correct clustering of all $N$ users while simultaneously obtain low regret in the process. Before providing the algorithm and regret guarantee for clustering let us define the (minimum) separation parameter as $\Delta = \min_{i,j; i\neq j}\|\theta^*_i -\theta^*_j\|$. Note that if $\Delta$ is large, the problem is easier to solve, and vice versa. Hence, $\Delta$ is also called the SNR (signal to noise ratio, since we assume that the noise variance is unity) of the problem.

\subsubsection{Algorithm and its regret bound}
\label{sec:algo_cluster}
We now present the clustering algorithm, formally given in Algorithm~\ref{alg:cluster}. We let all the agents play the learning algorithm \texttt{OFUL} of \citet{oful} for $n$ rounds. Note that, with a lower bound on the minimum eigenvalue of the Gram matrix (Theorem~\ref{thm:main-result}, Assumption~\ref{ass:high-prob}), we now can compute how close the estimated parameters $\{\hat{\theta}^{(i)}\}_{i=1}^N$ are to the true parameters in $\ell_2$ norm. The choice of norm is important here. Note that \texttt{OFUL}  also obtains an estimate of the underlying parameter. However, that closeness is only guaranteed in a problem dependent matrix norm, which can not be directly used for inference problems like clustering. Note that Theorem.. basically converts this problem dependent norm to (an universal) $\ell_2$ norm closeness guarantee, which enables us to perform clustering. After $n$ rounds, the center performs a pairwise clustering with threshold $\gamma$. If the pairwise distance between two estimates are less than $\gamma$ they are estimated to belong to the same cluster, otherwise different. This procedure is given in the \texttt{EDGE-CLUSTER} subroutine. With properly chosen threshold, we show that Algorithm~\ref{alg:cluster} clusters the agents correctly with high probability.
\begin{algorithm}[t!]
\caption{Cluster without Explore}
  \begin{algorithmic}[1]
    \STATE \textbf{Input:} No. of users $N$, learning horizon $n$, high probability parameter $\delta$, threshold $\eta$\\
    \vspace{1mm}
  {\textbf{ Individual Learning Phase}}    
\STATE All agents play \texttt{OFUL}($\delta)$ independently for $n$ rounds
    \STATE $\{\hat{\theta}^{(i)}\}_{i=1}^N \gets $ All agents' estimates at the end of round $n$ and send to the center\\
{\textbf{ Cluster the Users at center}}    
\STATE $\texttt{User-Cluster} \gets $ {\ttfamily EDGE-CLUSTER$(\{\hat{\theta}^{(i)}\}_{i=1}^N, \eta)$} \\
\vspace{1mm} 
{\underline{ {\ttfamily EDGE-CLUSTER}}}    
   \STATE  \textbf{Input:} All estimates $\{\hat{\theta}^{(i)}\}_{i=1}^N$,  threshold $\eta \geq 0$.
 \STATE Construct an undirected Graph $G$ on $N$ vertices as follows:
$
 || \widehat{\theta}^*_i - \widehat{\theta}^*_j || \leq \eta \Leftrightarrow  i \sim_G j
$
 \STATE $\mathcal{C} \gets \{C_1, \cdots, C_k \} $ all strongly connected components of $G$
 \STATE \textbf{Return :} $\mathcal{C}$
  \end{algorithmic}
  \label{alg:cluster}
\end{algorithm}

\begin{lemma}
\label{lem:cluster}
 Suppose we choose the threshold $\eta = \frac{4}{n^{1/4}}\sqrt{\frac{2d \log(n/\delta)}{\gamma \log(d/\delta)}}$, and the separation $\Delta$ satisfies
    $\Delta > 2\eta = \frac{8}{n^{1/4}} \sqrt{\frac{2d \log(n/\delta)}{\gamma \log^{1/2}(d/\delta)}}.$
 Then, Algorithm~\ref{alg:cluster} clusters all $N$ agents correctly with probability at least $1- 4 \binom{N}{2}\delta$. Furthermore, the regret of any agent $i \in [N]$ is given by $R_i \leq \mathcal{O}(d\sqrt{n}\log(1/\delta))$ with probability at least $1-\delta$.
 \end{lemma}
 
Note that the separation decays with $n$, and for a large $n$, this is just a mild requirement. Also, our clustering algorithm does not incur regret from pure exploration, and the regret is just from playing the OFUL algorithm.

 \begin{remark}[Clustering gain]
 We run Algorithm~\ref{alg:cluster} upto to the instant where all users are clustered correctly wit high probability. In particular, we do not characterize the clustering gain of Algorithm~\ref{alg:cluster} since this is not the main focus of the paper. In order to see the gain, one needs to do the following: after the agents are clustered, the center treats each cluster as a single agent, and averages reward from all users in the same cluster. This ensures standard deviation of the resulting noise goes down by a factor of $1/\sqrt{N'}$ (clustering gain, similar to \cite{gentile2014online,ghosh2021collaborative}), where $N'$ is the number of users in the cluster. 
 \end{remark}
\textbf{Clustering without separation assumption on $\Delta$.} In the above result, we assume that the separation $\Delta$ satisfies the above condition. Instead, if the learner knows $\Delta$, one can set
\begin{align*}
   \frac{n}{\log^2(n/\delta)} \geq C \frac{d^2}{\gamma^2 \, \Delta^4} \left(\frac{1}{\log(d/\delta)} \right) = \Tilde{\Omega}(\frac{d^2}{\gamma^2 \, \Delta^4}),
\end{align*}
and the threshold $\eta = \Delta/2$ to obtain the same result.

%% file: conclusion.tex
\section{CONCLUSION}

We present a minimum eigenvalue bound on the expected design matrix---a theoretical extension to what is already known for discrete action space. In particular, we show that in action spaces with locally "nice" curvature properties, the above-mentioned minimum eigenvalue grows at the order of $\sqrt{n}$. We show that this eigenvalue bound enables us to obtain inference and minimize regret simultaneously, in the linear bandit setup. We then apply our findings in two practical applications, \emph{bandit clustering} and \emph{model selection}. We emphasize these results pave the way for new research in the area of stability and robustness of practical bandit algorithms. Such ideas are not new and the study of differentially private algorithms \citep{dwork2014algorithmic} focuse on this area precisely. Another line of research is to study the question of asymptotic optimality for algorithms based on the principle of optimism in the continuous action domain. 

%% file: app_main.tex
\section{LOCALLY CONSTANT HESSIAN SURFACES (PROOF OF THEOREM~\ref{thm:main-result})}
\label{app:main}
In this section we shall prove Theorem~\ref{thm:main-result}. We shall divide the proof of the theorem in three parts. In Section~\ref{subsec:sphere} we shall prove it for spherical surface action sets. In Section~\ref{subsec:Ellipsoidal} we shall extend the proof for ellipsoidal surface action sets and finally in Section \ref{subsec:LCH} we prove the result for the generalization to action spaces which are LCH surfaces.

\subsection{Spherical Action Sets}
\label{subsec:sphere}

\begin{theorem}
\label{theorem:sphere}
Let the set of arms be $\cX = \cS^{d-1} := \left\{x \in \mathbb{R}^d: \| x\| = 1\right\}$, the surface of the $d$ dimensional unit sphere. Let $\bar{G}_n = \mathbb{E}_\theta\left[\sum_{s=1}^n A_sA_s^\top\right]$, where $\theta$ is a bandit parameter and $A_s$ are arms in $\cX$ drawn according to some bandit algorithm.
For any bandit algorithm which suffers expected regret, $R_n(\theta)$, at most $O(\sqrt{n})$, 
\begin{align*}
 \lambda_{min}(\bar{G}_n) = \Omega(\sqrt{n})\;.   
\end{align*}
That is there exists constants $\gamma > 0$ and a finite time $n_0$ such that for all $n \geq n_0$, we have $\lambda_{min}(\bar{G}_n) \geq \gamma\sqrt{n}$. 
\end{theorem}

\begin{proof}
\label{proof:sphere}
We start with a result which is standard while proving such lower bounds (\cite{lattimore2017end, lattimore2020bandit, hao2020adaptive}) which is using the measure change inequality of \cite{kaufmann2016complexity} (see Lemma \ref{lemma:Garivier_Kauffmann}) combined with the Divergence Decomposition Lemma under the Linear Bandit Setup \cite{lattimore2020bandit} (see Lemma \ref{lemma:Divergence_Decomposition}) to get 
\begin{align}
\label{eq:eq8}
\frac{1}{2}\|\theta - \theta'\|^2_{\bar{G}_n} \geq \mathrm{KL}(\mathrm{Ber}(\mathbb{E}_{\theta}[Z]) || \mathrm{Ber}(\mathbb{E}_{\theta'}[Z]))\; .    
\end{align}
 For any time $n$ we have an eigenvalue decomposition of $\bar{G}_n$ as  
\begin{align}
\label{eq:eq9}
\bar{G}_n = \sum_{i=1}^d \lambda_i u_i u_i^\top\;,    
\end{align}
where $\lambda_1 \geq \lambda_2 \geq \cdots, \geq \lambda_d \triangleq \lambda_{min}(\bar{G}_n)$.
From the regret property we have, for all $M$, such that for any bandit parameter $\theta$ satisfying $\norm{\theta} \leq 2M$, there exists a constant $c > 0$  such that expected regret $\expect{R_n(\theta)} \leq c\sqrt{n}$ for any $n$. Without loss of generality let us fix the bandit parameter $\theta = e_1$ (Otherwise we can always rename our coordinate system s.t. in the new system $\theta = e_1$) and $\theta' = \theta + \alpha u_d$ for some $\alpha$ to be decided later. This gives,
\begin{align}
\label{eq:eq10}
  \|\theta - \theta'\|^2_{\bar{G}_n} = \alpha^2 \lambda_d\;.  
\end{align}

\paragraph{$\epsilon$-neighbourhood optimal arms}

We shall define $\mathrm{OPT}_\epsilon(\theta)$ as the set of $\epsilon-$ suboptimal arms with respect to $\theta$
\begin{align*}
  \mathrm{OPT}_\epsilon(\theta) \triangleq \{ x \in \cX \mid x^T\theta \geq \sup_x x^T \theta - \epsilon \}  
\end{align*}
We shall with abuse of notation denote $\mathrm{OPT}_0(\theta) = \mathrm{OPT}(\theta) = e_1$. Let us also define $\mathrm{N}_{\epsilon, n}(\theta)$ to be the number of times in time $n$ that an arm in $\mathrm{OPT}_\epsilon(\theta)$ has been played.
\begin{align*}
  \mathrm{N}_{\epsilon, n}(\theta) \triangleq \sum_{s=1}^n\mathrm{1}\{A_s \in \mathrm{OPT}_\epsilon(\theta)\} . 
\end{align*}
Under the hypothesis that for any bandit parameter $\theta$ satisfying $\norm{\theta} \leq 2$, there exists a constant $c > 0$ such that $\mathbb{E}R_n(\theta) \leq c\sqrt{n}$, we have
\begin{align*}
     c\sqrt{n} \geq \mathbb{E}R_n(\theta) \geq \mathbb{E}[\sum_{s=1}^n \mathrm{1}\{A_s \notin \mathrm{OPT}_\epsilon(\theta)\}(\max_x x^T\theta - A_s^T\theta)]
     \newline
    \geq \eps(\mathbb{E}[n-\mathrm{N}_{\epsilon, n}(\theta)]) \;.
\end{align*}
Rearranging, we get $\mathbb{E}_\theta[\mathrm{N}_{\epsilon, n}(\theta)] \geq n - \frac{c\sqrt{n}}{\epsilon}$.

\paragraph{Showing that $u_d$ is approximately perpendicular to $\mathrm{OPT}(\theta)$.}

Let $v \equiv \mathrm{OPT}(\theta)$. Then, by definition $v = \theta = e_1$. Let us decompose $\bar{G}_n$ into $nvv^\top$, the unperturbed component, $A$, and $\bar{G}_n - nvv^\top$ the perturbation, $H$, as per our notation of the perturbation lemmas Davis-Kahan(see Lemma \ref{lemma:Davis_Kahan}) and Weyl's lemma (see Lemma \ref{lemma:Weyl's}).
\begin{align}
\label{eq:eq12}
  \bar{G}_n = \underbrace{nvv^\top}_{A}+ \underbrace{\bar{G}_n - nvv^\top}_{H, \text{"perturbation"}}  
\end{align}
Note that both the unperturbed matrix $A$, the perturbation matrix $H$ are hermitian. Thus we can use Weyl's Lemma (Lemma \ref{lemma:Weyl's}). This gives us for all $i=2,3,\cdots,d$
\begin{equation}
\label{eq:eq13}
  \lambda_i(\bar{G}_n) \leq \underbrace{\lambda_i(A)}_{=0} + \lambda_{\max}(H) , 
\end{equation}
where we have used $\lambda_i(e_1e_1^\top) = 0$ for all $i=2,3,\cdots,d$. Now note from the definition of the perturbation matrix $H$ and our notation for the spectral norm
\begin{align*}
    \lambda_{\max}(H) = \|\bar{G}_n - n vv^\top\|
    \newline
    =\norm{\expect{\sum_{s=1}^n A_sA_s^\top - vv^\top}},
\end{align*}
where in the last equation we have expanded the definition of $\bar{G}_n$.
We now decompose the last sum into two parts, all arms, $A_s$ which belong in the group $\mathrm{OPT}_\epsilon(\theta)$ and those that do not. We are doing this because we shall see that all arms that belong in the group of $\mathrm{OPT}_\epsilon(\theta)$ will have norm $\|A_s - v \|$ small than those which do not. We get
\begin{equation}
\label{eq:eq14}
    \lambda_{max}(H)=\norm{\expect{\sum_{s: A_s \in \mathrm{OPT}_\epsilon(\theta) }(A_sA_s^\top - vv^\top) + \sum_{s: A_s \notin \mathrm{OPT}_\epsilon(\theta) }(A_sA_s^\top - vv^\top)}}.
\end{equation}
Now, by interchanging norm and expectation and using triangle inequality for norms we have that
\begin{equation}
\label{eq:eq15}
    \lambda_{max}(H)\leq \sum_{s: A_s \in \mathrm{OPT}_\epsilon(\theta) }\expect{\norm{A_sA_s^\top - vv^\top}} + \sum_{s: A_s \notin \mathrm{OPT}_\epsilon(\theta) }\expect{\norm{(A_sA_s^\top - vv^\top)}} \;.
\end{equation}
Thus, we see that we need to control the norm of $A_sA_s^\top - vv^\top$ for two separate cases. Before we do that, let us first rearrange $A_sA_s^\top - vv^\top$
to get 
\begin{align}
    \|A_sA_s^\top - vv^\top\| = \|(A_s - v)A_s^\top + v(A_s - v)^\top\| \leq \|A_s - v\|(\|A_s\| + \|v\|) 
\leq 2\|A_s - v\| \label{eq:eq16}
\end{align}
where in the last equation we have used the fact that both $A_s$ and $v$ belongs to $\cS^{d-1}$. Note that here we need to have kept $\cS^{d-1}$ and a generic upper bound on the arm set $\cX$ would have sufficed. Thus, we need to control $\|A_s - v\|$ in the two cases. When $A_s \notin \mathrm{OPT}_\epsilon(\theta)$ we can use a generic upper bound on the bounded property of the Action Space to get $\|A_s - v\| \leq 2$. Now, let us consider the case when $A_s \in \mathrm{OPT}_\epsilon(\theta)$. Then by definition of $\mathrm{OPT}_\epsilon(\theta)$ we have
\begin{align}
\label{eq:eq17}
 A_s^\top\theta \geq v^\top\theta - \epsilon = 1- \epsilon   
\end{align}
where we have again used that $v = e_1 = \theta$ because of the action-space $\cX$. Note here again that for a general geometry $v^\top \theta$ would still be some constant depending upon the geometry of the action-space.

Now let us consider the 2 orthogonal components of such an $A_s$,
$A_s - (A_s^\top v)v$ and $(A_s^\top v)v$.
We have by Pythagorean Theorem
\begin{align*}
  \|A_s\|^2 = \|(A_s^\top v)v\|^2 + \|A_s - (A_s^\top v)v\|^2 . 
\end{align*}
Therefore again utilizing that $\|\|A_s\|^2 = 1$ (although a generic bound would have sufficed) and rearranging
\begin{align*}
    \|A_s - (A_s^\top v)v\|^2 = 1 - \|(A_s^\top v)v\|^2.
\end{align*}
Now we utilize the fact that in the spherical geometry we can exactly compute $A_s^\top v = A_s^\top \theta \geq 1 -\epsilon$, (see Equation~\eqref{eq:eq17}), to get 
\begin{align*}
  1 - \|(A_s^\top v)v\|^2  \leq 1 - (1-\epsilon)^2 = (2\epsilon - \epsilon^2)\;,
\end{align*} where we have used homogenity of norms in the inequality.
For other geometry we need to use the geometry of the set $\mathrm{OPT}_\epsilon(\theta)$ to estimate $A_s^\top v$.

Now, we can conclude that for any $A_s \in \mathrm{OPT}_\epsilon(\theta)$, by virtue of orthogonality between $A_s - (A_s^\top v)v$ and $v -(A_s^\top v)v$, we have
using the Pythagorean theorem 
\begin{align*}
  \|A_s - v\|^2 = \|A_s - (A_s^\top v)v\|^2 + \|v -(A_s^\top v)v\|^2.  \end{align*}
Now using the fact that $v -(A_s^\top v)v = (1-A_s^\top v)v \leq \epsilon v$, we get
\begin{align*}
    \|A_s - (A_s^\top v)v\|^2 + \|v -(A_s^\top v)v\|^2 \leq (2\epsilon - \epsilon^2) + \epsilon^2 = 2\epsilon\;.
\end{align*}
Thus we have for any $A_s \in \mathrm{OPT}_\epsilon(\theta)$ the following estimate
\begin{equation}
\label{eq:eq18}
    \|A_s - v\| \leq \sqrt{2\epsilon}~.
\end{equation}
Note that by changing the geometry we will have different constants but still of the same $\sqrt{\epsilon}$ order.

Thus we have from Equations ~\eqref{eq:eq15}, \eqref{eq:eq16} and \eqref{eq:eq18} along with the trivial bound of $\|A_s - v\| \leq 2$, the following estimate on the spectral norm of the perturbation matrix $H$ : 
\begin{align*}
  \lambda_{max}(H) \leq 2\sqrt{2\epsilon}\mathbb{E}_{\theta}[\mathrm{N}_{\epsilon, n}(\theta)] + 4\mathbb{E}_{\theta}[n - \mathrm{N}_{\epsilon, n}(\theta)] \;,
\end{align*}
where we have used the definitions of $\mathrm{N}_{\epsilon, n}(\theta) = \sum_{s=1}^n \mathrm{1}\{A_s \in \mathrm{OPT}_\epsilon(\theta)\}$ and $n - \mathrm{N}_{\epsilon, n}(\theta) = \sum_{s=1}^n \mathrm{1}\{A_s \notin \mathrm{OPT}_\epsilon(\theta)\} $.
Then using a crude upper bound of $n$ on $\mathbb{E}_{\theta}[\mathrm{N}_{\epsilon, n}(\theta)]$ and using the expression for the lower bound of $\mathbb{E}_{\theta}[\mathrm{N}_{\epsilon, n}(\theta)]$ to get an upper bound on $\mathbb{E}_{\theta}[n - \mathrm{N}_{\epsilon, n}(\theta)]$, we get
\begin{align}
\label{eq:eq19}
  \leq 2\sqrt{2\epsilon}n + 4\frac{c\sqrt{n}}{\epsilon} \;. 
\end{align}

Till now $\epsilon$ was a free parameter. We choose $\epsilon = \frac{c}{0.01\sqrt{n}}$ such that, $\frac{c\sqrt{n}}{\epsilon} = 0.01n$.
This gives from Equation \eqref{eq:eq19}
\begin{align}
\label{eq:crux}
  \lambda_{max}(H) \leq 2\sqrt{2}\sqrt{\frac{c}{0.01}}n^{\frac{3}{4}} + 4*0.01*n ~. 
\end{align}
Thus for a large but finite $n$, depending upon $c$ and also on the geometry of the arm set, we have  
\begin{equation}
\label{eq:eq20}
    \lambda_{max}(H) \leq 0.1n~.
\end{equation}
Hence, from Equation \eqref{eq:eq13}, we have
\begin{equation}
\label{eq:eq21}
    \lambda_i(\bar{G}_n) \leq 0.1 n \text{ }\forall i = 2,3,\ldots,d \; .
\end{equation}

Thus we have from Equation \eqref{eq:eq20} the norm of the perturbation matrix; from equation \eqref{eq:eq9} we have the eigen decompostion of $\bar{G}_n$; we also have the eigen-decomposition of $n vv^\top$ as $nvv^\top = nvv^\top + \sum_{i=2}^d 0 \tilde{u}_i\tilde{u}_i^\top$ and from equation \eqref{eq:eq21} we have the eigen-gap $\delta = \lambda_1(nvv^\top) - \lambda_2(\bar{G}_n) \geq n - 0.1n = 0.9n$.

Therefore the Davis Kahan Theorem, (see Theorem \ref{lemma:Davis_Kahan}) gives us
\begin{equation*}
\|v^T \begin{bmatrix}
u_2 & u_3 & \cdots & u_d\\
\end{bmatrix}\| \leq \frac{\|H\|}{n-0.1n} \leq \frac{0.1n}{0.9n} = \frac{1}{9}  \;,  \end{equation*}
which implies
\begin{align*}
\max_{i \in 2,\ldots, d} v^Tu_i \leq \frac{1}{9},    
\end{align*}
which we interpret as the eigenvector corresponding to $\lambda_{min}(\bar{G}_n)$ is sufficiently orthogonal to $\mathrm{OPT}(\theta)$ when the optimal set decreases as roughly $\frac{1}{\sqrt{n}}$.

\paragraph{The amount $\alpha$ to perturb $\theta$ to $\theta'$ }
In this penultimate section, we will try to find $\alpha$ such that the $\epsilon$-optimal arms for $\theta$ and $\theta'$ are necessarily disjoint. 

Formally we would want to find the smallest $\alpha$ s.t. $$\mathrm{OPT}_\epsilon(\theta) \cap \mathrm{OPT}_\epsilon(\theta + \alpha u_d) = \emptyset$$ 
subject to the constraint $|\langle u_d, \mathrm{OPT}(\theta) \rangle| \leq \frac{1}{9}$
 
 \begin{lemma}
 \label{lemma:step}
  There exists a constant $\beta$ such that $\mathrm{OPT}_\epsilon(\theta) \cap \mathrm{OPT}_\epsilon(\theta + \frac{\beta}{n^{1/4}} u_d) = \emptyset$ with $\epsilon$ defined as above with $u_d$ is subject to the constraint $|\langle u_d, \mathrm{OPT}(\theta) \rangle| \leq \frac{1}{9}$.
 \end{lemma}

\begin{proof}


We will be focusing on $3$ points $\theta$, $\theta + \alpha u_d$ and the centre of the sphere. Since a plane can be drawn passing through any 3 non-collinear points we will be working in terms of the triangle formed by these three points. We will refer to Figure \ref{img:proof} for the proof.

Let us first note that for any given arbitrary $\epsilon$, the $\mathrm{OPT}_\epsilon(\theta)$ subtends an angle $\psi$ about the centre. (See left image of Figure \ref{img:proof}). In order to ensure that $\mathrm{OPT}_\epsilon(\theta)$ and $\mathrm{OPT}_\epsilon(\theta + \alpha u_d)$ are disjoint, we need to ensure that angular displacement of $\theta'$ is at least twice that is $2\psi$ from $\theta$ (see left image of Figure \ref{img:proof}) and thus choose the step size $\alpha$ based on the angular displacement $\psi$.  

\tikzset{every picture/.style={line width=0.75pt}}

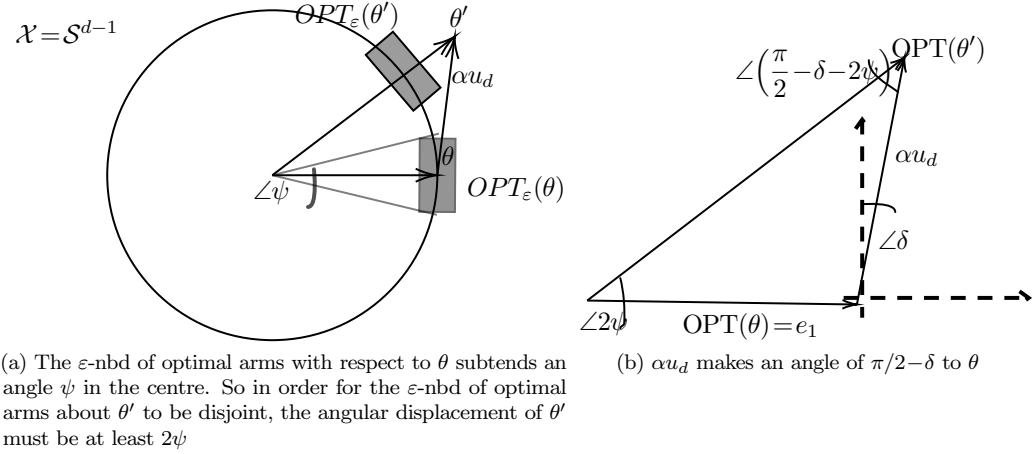
\begin{figure*}[t!]
\centering
\subfloat[][The $\epsilon$-nbd of optimal arms with respect to $\theta$ subtends an angle $\psi$ in the centre. So in order for the $\epsilon$-nbd of optimal arms about $\theta'$ to be disjoint, the angular displacement of $\theta'$ must be at least $2\psi$   ]{
\begin{tikzpicture}[x=0.75pt,y=0.75pt,yscale=-1,xscale=1]

\draw    (243,168) -- (324.25,168) ;
\draw [shift={(326.25,168)}, rotate = 180] [color={rgb, 255:red, 0; green, 0; blue, 0 }  ][line width=0.75]    (10.93,-3.29) .. controls (6.95,-1.4) and (3.31,-0.3) .. (0,0) .. controls (3.31,0.3) and (6.95,1.4) .. (10.93,3.29)   ;
\draw    (326.25,168) -- (334.75,99.98) ;
\draw [shift={(335,98)}, rotate = 97.13] [color={rgb, 255:red, 0; green, 0; blue, 0 }  ][line width=0.75]    (10.93,-3.29) .. controls (6.95,-1.4) and (3.31,-0.3) .. (0,0) .. controls (3.31,0.3) and (6.95,1.4) .. (10.93,3.29)   ;
\draw    (243,168) -- (333.41,99.21) ;
\draw [shift={(335,98)}, rotate = 142.73] [color={rgb, 255:red, 0; green, 0; blue, 0 }  ][line width=0.75]    (10.93,-3.29) .. controls (6.95,-1.4) and (3.31,-0.3) .. (0,0) .. controls (3.31,0.3) and (6.95,1.4) .. (10.93,3.29)   ;
\draw   (159.75,168) .. controls (159.75,122.02) and (197.02,84.75) .. (243,84.75) .. controls (288.98,84.75) and (326.25,122.02) .. (326.25,168) .. controls (326.25,213.98) and (288.98,251.25) .. (243,251.25) .. controls (197.02,251.25) and (159.75,213.98) .. (159.75,168) -- cycle ;
\draw  [fill={rgb, 255:red, 0; green, 0; blue, 0 }  ,fill opacity=0.39 ] (290.15,107.25) -- (303.85,95.58) -- (327.85,123.75) -- (314.15,135.42) -- cycle ;
\draw  [color={rgb, 255:red, 0; green, 0; blue, 0 }  ,draw opacity=0.58 ][fill={rgb, 255:red, 0; green, 0; blue, 0 }  ,fill opacity=0.42 ] (317.39,149.43) -- (335.39,149.57) -- (335.11,186.57) -- (317.11,186.43) -- cycle ;
\draw [color={rgb, 255:red, 0; green, 0; blue, 0 }  ,draw opacity=0.51 ][line width=0.75]    (243,168) -- (326,188) ;
\draw [color={rgb, 255:red, 0; green, 0; blue, 0 }  ,draw opacity=0.51 ][line width=0.75]    (243,168) -- (327,147) ;
\draw  [color={rgb, 255:red, 0; green, 0; blue, 0 }  ,draw opacity=0.69 ][line width=1.5] [line join = round][line cap = round] (264,165) .. controls (264,165.05) and (266.09,184) .. (261,184) ;

\draw (253,79) node [anchor=north west][inner sep=0.75pt]   [align=left] {$\displaystyle OPT_{\epsilon }( \theta ')$};
\draw (331,81) node [anchor=north west][inner sep=0.75pt]   [align=left] {$\displaystyle \theta '$};
\draw (327,152) node [anchor=north west][inner sep=0.75pt]   [align=left] {$\displaystyle \theta $};
\draw (332,114) node [anchor=north west][inner sep=0.75pt]   [align=left] {$\displaystyle \alpha u_{d}$};
\draw (340,167) node [anchor=north west][inner sep=0.75pt]   [align=left] {$\displaystyle OPT_{\epsilon }( \theta )$};
\draw (112,88) node [anchor=north west][inner sep=0.75pt]   [align=left] {$\displaystyle \mathcal{X} =\mathcal{S}{^{d-1}}$};
\draw (232,169) node [anchor=north west][inner sep=0.75pt]   [align=left] {$\displaystyle \angle \psi $};

\end{tikzpicture}}
\subfloat[][$\alpha u_d$ makes an angle of $\pi/2-\delta$ to $\theta$]{
\begin{tikzpicture}[x=0.50pt,y=0.50pt,yscale=-1,xscale=1]

\draw    (154,232) -- (356,234.97) ;
\draw [shift={(358,235)}, rotate = 180.84] [color={rgb, 255:red, 0; green, 0; blue, 0 }  ][line width=0.75]    (10.93,-3.29) .. controls (6.95,-1.4) and (3.31,-0.3) .. (0,0) .. controls (3.31,0.3) and (6.95,1.4) .. (10.93,3.29)   ;
\draw    (358,235) -- (393.62,49.96) ;
\draw [shift={(394,48)}, rotate = 100.9] [color={rgb, 255:red, 0; green, 0; blue, 0 }  ][line width=0.75]    (10.93,-3.29) .. controls (6.95,-1.4) and (3.31,-0.3) .. (0,0) .. controls (3.31,0.3) and (6.95,1.4) .. (10.93,3.29)   ;
\draw    (154,232) -- (392.41,49.22) ;
\draw [shift={(394,48)}, rotate = 142.52] [color={rgb, 255:red, 0; green, 0; blue, 0 }  ][line width=0.75]    (10.93,-3.29) .. controls (6.95,-1.4) and (3.31,-0.3) .. (0,0) .. controls (3.31,0.3) and (6.95,1.4) .. (10.93,3.29)   ;
\draw [dash pattern={on 5.63pt off 4.5pt}][line width=1.5]  (348,230) -- (488,230)(362,95) -- (362,245) (481,225) -- (488,230) -- (481,235) (357,102) -- (362,95) -- (367,102)  ;
\draw  [color={rgb, 255:red, 0; green, 0; blue, 0 }  ][line width=0.75] [line join = round][line cap = round] (179,215) .. controls (181.76,226.04) and (182.49,237.63) .. (182,249) .. controls (181.9,251.27) and (179,251.1) .. (179,253) ;
\draw  [color={rgb, 255:red, 0; green, 0; blue, 0 }  ][line width=0.75] [line join = round][line cap = round] (363,159) .. controls (363.88,159) and (387,157.24) .. (387,165) ;
\draw  [color={rgb, 255:red, 0; green, 0; blue, 0 }  ][line width=0.75] [line join = round][line cap = round] (389,74) .. controls (380.06,67.3) and (367,61.17) .. (367,50) ;

\draw (385,116) node [anchor=north west][inner sep=0.75pt]   [align=left] {$\displaystyle \alpha u_{d}$};
\draw (224,239) node [anchor=north west][inner sep=0.75pt]   [align=left] {$\displaystyle \mathrm{OPT}( \theta ) =e_{1}$};
\draw (382,29) node [anchor=north west][inner sep=0.75pt]   [align=left] {$\displaystyle \mathrm{OPT}( \theta ')$};
\draw (146,238) node [anchor=north west][inner sep=0.75pt]   [align=left] {$\displaystyle \angle 2\psi $};
\draw (372,177) node [anchor=north west][inner sep=0.75pt]   [align=left] {$\displaystyle \angle \delta $};
\draw (264,39) node [anchor=north west][inner sep=0.75pt]   [align=left] {$\displaystyle \angle \left(\frac{\pi }{2} -\delta -2\psi \right)$};

\end{tikzpicture}}
\caption{Proof of Lemma \ref{lemma:step}}
\label{img:proof}
\end{figure*}
Note that even in continuous action-sets with arbitrary geometry, if the geometry is smooth, the idea still remains the same. Thus we have, 
\begin{align}
\label{eq:eq22}
    \psi = \cos^{-1}(1-\epsilon)
\end{align} (see the left image of Figure \ref{img:proof} for better clarity $\langle x, \theta \rangle = \cos \psi \geq 1-\epsilon$ for any $x \in \mathrm{OPT}_\epsilon(\theta)$ ) and hence the displacement must be at least $2\cos^{-1}(1-\epsilon)$. 

Now from the constraint condition that $v^\top u_d \leq 1/9$, we refer to the right hand side of figure $\ref{img:proof}$, where $\delta$ is the complementary angle between by $v$ (and hence $\theta$) and $u_d$.
Thus \begin{align}
\label{eq:eq23}
    \sin \delta = \cos(\pi/2 - \delta) = \langle v, u_d \rangle\leq 1/9
\end{align}  (see the right image of Figure \ref{img:proof}) and hence $\cos \delta \geq \sqrt{\frac{80}{81}}$. 
Therefore applying the law of sines we have (see the right image of Figure \ref{img:proof} )
\begin{align}
\label{eq:eq24}
\frac{\alpha}{\sin 2\psi} = \frac{1}{\sin (\pi/2 - \delta - 2\psi)} \implies \alpha = \frac{\sin(2\psi)}{\sin (\pi/2 - \delta - 2\psi)} = \frac{\sin(2\psi)}{\cos (\delta + 2\psi)}
\end{align}
where $\alpha$ is the step-size to be determined.

Now note from left image of Figure \ref{img:proof} that we have (see Equation \eqref{eq:eq22})
\begin{align*}
    \cos \psi = 1-\epsilon \implies \sin \psi = \sqrt{2\epsilon-\epsilon^2} \implies \cos^{-1}(1-\epsilon) = \sin^{-1}(\sqrt{2\epsilon-\epsilon^2}) \approx \sqrt{2\epsilon - \epsilon^2} \approx\sqrt{2\epsilon}
\end{align*}
for small angles and therefore $\sin 2\psi \approx 2\psi = 2\cos^{-1}(1-\epsilon)\approx 2\sqrt{2\epsilon}$
and $\cos  (\delta + 2\psi) \approx \cos \delta \geq \sqrt{\frac{80}{81}} $ (see Equation \eqref{eq:eq23}).
Therefore from Equation \eqref{eq:eq24} we have
\begin{align*}
    \alpha \lessapprox \frac{2\sqrt{2\epsilon}}{ \sqrt{\frac{80}{81}}} =c_2\sqrt{\epsilon} =\beta \frac{1}{n^{\frac{1}{4}}},
\end{align*}
and hence the claim is proved.
\end{proof}

\paragraph{Finishing the proof}\label{par:finishing_the_proof}
Let us now define the random variable $Z$ in Equation \eqref{eq:eq8} as the fraction of times an arm in the $\epsilon$-optimal set has been played, that is $Z \triangleq \frac{\mathrm{N}_{\epsilon, n}(\theta)}{n}$. Then , we have 
\begin{align*}
    \mathbb{E}_\theta[Z] \geq \frac{n - \frac{c\sqrt{n}}{\epsilon}}{n} = \frac{0.99n}{n} = 0.99,
\end{align*}
where we have used the definition of $Z$, the expression for the lower bound on $\mathbb{E}_\theta[\mathrm{N}_{\epsilon, n}(\theta)]$ and the value of $\epsilon$.
On the other hand, because of the disjoint $\epsilon$-optimal sets of $\theta$ and $\theta'$ by construction, we have
\begin{align*}
    \mathbb{E}_{\theta'}[Z] = \frac{ \mathbb{E}_{\theta'}[\mathrm{N}_{\epsilon, n}(\theta)]}{n} = \frac{ n - \mathbb{E}_{\theta'}[\mathrm{N}_{\epsilon, n}(\theta')]}{n}\leq \frac{\frac{c\sqrt{n}}{\epsilon}}{n} =\frac{0.01n}{n} = 0.01
\end{align*}
where we have used the regret property of the algorithm of being at most $c\sqrt{n}$ for any bandit parameter satisfying $\norm{\theta} \leq 2$ by having $n$ large enough such that $\norm{\theta'} = \norm{\theta + \alpha u_d} \leq 2$   .
Thus using Equations \eqref{eq:eq8} and \eqref{eq:eq10} and using the estimates of $\mathbb{E}_{\theta'}[Z]$ and $\mathbb{E}_\theta[Z]$ we have
\begin{align*}
    \frac{\alpha^2}{2} \lambda_d \geq \mathrm{KL}(\mathrm{Ber}(\mathbb{E}_{\theta}[Z]) || \mathrm{Ber}(\mathbb{E}_{\theta'}[Z])) \geq 2(\mathbb{E}_{\theta}[Z] - \mathbb{E}_{\theta'}[Z])^2 \geq 2 (0.99-0.01)^2 = c_4
\end{align*}
where for the second inequality we have used Pinsker's Inequality. Finally using Lemma \ref{lemma:step} to find $\alpha$ which also guarantees our estimates of $\mathbb{E}_{\theta'}[Z]$ and $\mathbb{E}_\theta[Z]$, we have
\begin{align*}
    \lambda_d(\bar{G}_n)) \geq 2 c_4 \frac{1}{\alpha^2} = c_5 \frac{1}{(\frac{\beta}{n^\frac{1}{4}})^2}=c_6\sqrt{n}
\end{align*} 
This completes the proof with the positive constant $c_6$ being redefined as $\gamma$.
\end{proof}

\begin{remark}[Dependence on dimension $d$]
Though our result does not explicitly depend on the dimension $d$, it can depend on it through the regret upper bound of underlying linear bandit algorithm. Hypothetically, if there exists an algorithm which does not depend upon $d$, our result shows that the growth of minimum eigenvalue of design matrix is also independent of $d$. That being said, the dimensional dependency enters through the constant $c$ and is different for different algorithms. For example, for the OFUL algorithm \citep{oful}, we have $c = O(d)$. Similarly, for Thompson Sampling algorithm \citep{abeille2017linear}, we get $c = O(d^{3/2})$. 

We can show that $n_0 = \Omega(c^2)$ and $\gamma = \Omega(\frac{1}{c})$. Let us compute this for the spherical action set. From Equation \eqref{eq:crux}, we have 
\begin{equation*}
    \lambda_{\max}(H) \leq 2\sqrt{2}\sqrt{\frac{c}{0.01}}n^{\frac{3}{4}} + 4\times 0.01\times n.
\end{equation*}
Now, for a large but finite $n$, we have the right hand side to be less than $0.1n$. This finite $n$ is the $n_0$ as defined in Thm 2.3. An easy calculation shows that $n_0 = \Omega(c^2)$. Similarly, $\gamma$ can be computed as follows. Note that $\lambda_d \geq \frac{3.8416}{\alpha^2}= \frac{3.8416}{\beta^2} \sqrt{n}$, where $\beta = \frac{2\sqrt{\frac{2c}{0.01}}}{\sqrt{\frac{80}{81}}}$and $\gamma= \frac{3.8416}{\beta^2}$. This implies $\gamma=\Omega(\frac{1}{c})$.
    
\end{remark}

\subsection{Ellipsoidal Action-Set}
\label{subsec:Ellipsoidal}
\begin{theorem}
\label{theorem:ellipsod}
Let the set of arms be $\cX = := \left\{x \in \mathbb{R}^d: x^\top A^{-1}x = 1\right\}$, the surface of the $d$ dimensional ellipsoid ($A$ is symmetric and positive definite). Let $\bar{G}_n = \mathbb{E}_\theta\left[\sum_{s=1}^n A_sA_s^\top\right]$, where $\theta$ is a bandit parameter and $A_s$ are arms in $\cX$ drawn according to some bandit algorithm.

For any bandit algorithm which suffers expected regret, $R_n(\theta)$, at most $O(\sqrt{n})$, 
\begin{align*}
 \lambda_{min}(\bar{G}_n) = \Omega(\sqrt{n})\;.   
\end{align*}
That is there exists constants $\gamma > 0$ and a finite time $n_0$ such that for all $n \geq n_0$, we have $\lambda_{min}(\bar{G}_n) \geq \gamma\sqrt{n}$.
\end{theorem}

\begin{remark}
The constants $n_0$ and $\gamma$ depend upon the algorithm constants hidden by $O()$, the condition number of $A$ and the size of the bandit parameter $\theta$.    
\end{remark}

\begin{proof}
\label{proof:ellipsoid}
Let $\mathcal{X} = \{x\in \mathbb{R}^d : x^\top A^{-1} x =1\}$ be an ellipsoidal action set where $A$ is symmetric positive definite and $\theta$ be the bandit parameter. 

Let us make a change of variable $y = A^{-1/2}x$ for all $x \in \cX$. Thus we have $\norm{y}^2 = 1$. We define this new transformed action space $\cY$ and note that $\cY$ is the unit sphere.

Before we proceed further we make the following observations :
\begin{align}
\label{eq:El0}
\mathrm{OPT}_{\cX}(\theta) = \arg\max_{x \in \cX}\langle x,  \theta \rangle = \arg\max_{x \in \cX}\langle A^{-1/2}x,  A^{1/2}\theta \rangle 
\end{align}
\begin{align*}
   = A^{1/2}\arg\max_{y \in A^{-1/2}\cX}\langle y,A^{1/2}\theta\rangle
   = A^{1/2}\arg\max_{y \in \cY}\langle y,A^{1/2}\theta\rangle = A^{1/2}\mathrm{OPT}_{\cY}(A^{1/2}\theta).
\end{align*}
Similarly we have 
\begin{align}
\label{eq:El1}
    \mathrm{OPT}_{\epsilon, \cX}(\theta) = A^{1/2}\mathrm{OPT}_{\epsilon, \cY}(A^{1/2}\theta)
\end{align}

Thus let us consider the following linear bandit problem where the action set is $\cY = \cS^{d-1}$, and the bandit parameter is $A^{1/2}\theta$. Let the actions sampled be $\{y_i\}_{i=1}^n$ by any bandit algorithm with regret at most $O(\sqrt{n})$. Let the design matrix be $V_n = \expect{\sum_{i=1}^n y_iy_i^\top}$, and the corresponding eigenvector decomposition be $\sum_{i=1}^d\lambda_i u_iu_i^\top$.

From the previous section (see Lemma \ref{lemma:step}) we know,
\begin{align}
    \mathrm{OPT}_{\epsilon, \cY}\bigg(\frac{A^{1/2}\theta}{\norm{A^{1/2}\theta}}\bigg) \cap \mathrm{OPT}_{\epsilon, \cY}\bigg(\frac{A^{1/2}\theta}{\norm{A^{1/2}\theta}} + \alpha u_d\bigg) = \emptyset\;,
\end{align} for $\alpha = O(1/n^{1/4})$.
Thus multiplying the two disconnected sets by $\norm{A^{1/2}\theta}$ is still going to keep them disconnected and hence,
\begin{align}
    \mathrm{OPT}_{\epsilon, \cY}(A^{1/2}\theta) \cap \mathrm{OPT}_{\epsilon, \cY}\Big(A^{1/2}\theta + \alpha\norm{A^{1/2}\theta} u_d\Big) = \emptyset.
\end{align}
This implies, 
\begin{align}
    \mathrm{OPT}_{\epsilon, \cY}(A^{1/2}\theta) \cap \mathrm{OPT}_{\epsilon, \cY}\Big(A^{1/2}(\theta + \alpha\norm{A^{1/2}\theta}A^{-1/2} u_d)\Big) = \emptyset
\end{align}
and thus from the change of variable relation,(See equation \eqref{eq:El1}), we have
\begin{align}
    A^{-1/2}\mathrm{OPT}_{\epsilon, \cX}(\theta) \cap A^{-1/2}\mathrm{OPT}_{\epsilon, \cX}(\theta + \alpha\norm{A^{1/2}\theta}A^{-1/2} u_d) = \emptyset.
\end{align}
Now because continuous functions preserve disjoint sets, we have
\begin{align}
    \mathrm{OPT}_{\epsilon, \cX}(\theta) \cap \mathrm{OPT}_{\epsilon, \cX}(\theta + \alpha\norm{A^{1/2}\theta}A^{-1/2} u_d) = \emptyset.
\end{align}
Now let us define the design matrix $\bar{G}_n$ for the action set $\cX$ as $\bar{G}_n = \expect{\sum_{i=1}^n x_i x_i^\top}$.
With the change of variables we have
\begin{align}
\label{eq:El2}
    \bar{G}_n = \expect{\sum_{i=1}^n x_i x_i^\top}
    = A^{1/2}\expect{\sum_{i=1}^n y_i y_i^\top} A^{1/2} = A^{1/2}\sum_{i=1}^d\lambda_iu_iu_i^\top A^{1/2}\;,
\end{align}
where $u_i$ are the eigen-vectors as defined for the action space $\cY$.
Now defining the perturbed $\theta' = \theta + \alpha\norm{A^{1/2}\theta}A^{-1/2} u_d$, we have 
\begin{align}
    \|\theta - \theta'\|^2_{\bar{G}_n} = 
    \alpha^2\norm{A^{1/2}\theta}^2u_d^\top A^{-1/2} \Big(A^{1/2}\sum_{i=1}^d\lambda_iu_i u_i^T A^{1/2}\Big) A^{-1/2}u_d = \alpha^2\norm{A^{1/2}\theta}^2\lambda_d
\end{align}
Now, from the methodology discussed in the previous section using the Kauffman measure change inequality \ref{lemma:Garivier_Kauffmann}( see the paragraph \ref{par:finishing_the_proof}), we have 
\begin{align}
    \lambda_d \geq \frac{c}{\alpha^2\norm{A^{1/2}\theta}^2}=\frac{c\sqrt{n}}{\norm{A^{1/2}\theta}^2}\;,
\end{align} for some positive $c$.
To conclude the proof we need to find a relation between the eigenvalues lowest eigenvalue $\mu_d$ of $\bar{G}_n$ and the lowest eigenvalue $\lambda_d$ of $V_n$. For this note, from equation \eqref{eq:El2} 
\begin{align}
\label{eq:El3}
    \bar{G}_n = A^{1/2}V_n A^{1/2} = A^{1/2}U\Lambda U^\top A^{1/2}.
\end{align}
From spectral theory we have \begin{align}
    \mu_d(\bar{G}_n) = \frac{1}{\norm{\bar{G}_n^{-1}}}.
\end{align}
Now from equation \eqref{eq:El3}, and definition of matrix norms and orthogonality of eigenctors $U$, we have
\begin{align}
    \norm{G_n^{-1}} \leq \norm{A^{-1/2}}^2\norm{\Lambda^{-1}} ,
\end{align} and therefore,
\begin{align}
    \mu_d(\bar{G}_n) \geq \frac{1}{\norm{A^{-1/2}}^2\norm{\Lambda^{-1}}}=\frac{\lambda_d}{\norm{A^{-1/2}}^2} \geq \frac{c\sqrt{n}}{\norm{A^{-1/2}}^2\norm{A^{1/2}\theta}^2} \geq \frac{c\sqrt{n}}{\norm{A^{-1/2}}^2\norm{A^{1/2}}^2\norm{\theta}^2}.
\end{align}
This concludes the proof.
\end{proof}

In the next subsection we illustrate that the same proof for ellipsoids can be done through first principles. 

\subsubsection{Ellipsoidal Sets : Proof by first principles}
\label{subsubsec:Ellipoidal}
In this section we highlight the main ideas of the proof and leave the gaps as an exercise. The proof follows verbatim from the proof for the spherical case (see subsection \ref{subsec:sphere}).

\begin{remark}
We present this alternative proof to highlight the main takeaway of the proof technique, namely to construct an alternative bandit parameter $\theta'$, for which the actions played by the algorithm would significantly differ from the original bandit parameter and yet for both parameters, the algorithm plays optimally.
\end{remark}
\vspace{ 2mm}
\begin{remark}
The idea of this proof is also to highlight that what really is needed is information about neighbourhood of the optimal arm, and the global action space does not influence the exploration strategy of good regret algorithms. This would be useful for the proof of the LCH surfaces as given in the next Section~\ref{subsec:LCH}. 
\end{remark}
\begin{proof}
We fix a $\theta$ and start with the Garivier-Kauffman inequality 
\begin{align}
\label{eq:ellips_eq8}
\frac{1}{2}\|\theta - \theta'\|^2_{\bar{G}_n} \geq \mathrm{KL}(\mathrm{Ber}(\mathbb{E}_{\theta}[Z]) || \mathrm{Ber}(\mathbb{E}_{\theta'}[Z]))\; .    
\end{align}
and do an eigen-decomposition of $\bar{G}_n$ as  
\begin{align}
\label{eq:ellips_eq9}
\bar{G}_n = \sum_{i=1}^d \lambda_i u_i u_i^\top\;,  \end{align}
where $\lambda_1, \cdots, \lambda_d$ and $u_1, \cdots, u_d$ have the usual meaning.
The crux of the proof remains that we design the perturbed $\theta'$ as before, namely as $\theta' = \theta + \alpha u_d$ for some $\alpha$ to be computed such that $\mathrm{OPT}_\epsilon(\theta)$ and $\mathrm{OPT}_\epsilon(\theta')$ are disjoint. With this choice of $\theta'$ we have from the measure change inequality as,
\begin{align}
\label{eq:ellipse_eq10}
  \|\theta - \theta'\|^2_{\bar{G}_n} = \alpha^2 \lambda_d\;.  
\end{align}

As before, we have from the definitions of $\mathrm{OPT}_\epsilon(\theta)$ and $\mathrm{N}_{\epsilon, n}(\theta)$, 
\begin{gather}
\label{eq:estimate}
    \mathbb{E}_{\theta'}[\mathrm{N}_{\epsilon, n}(\theta')] \geq n - \frac{c\sqrt{n}}{\epsilon}\;,
\end{gather}
for all $\theta'$ satisfying $\norm{\theta'} \leq 2\norm{\theta}$ such that $\expect{R_n(\theta')} \leq c\sqrt{n}$, for some $c > 0$.
For ellipse we have the $\mathrm{OPT}(\theta) = \frac{A\theta}{\norm{\theta}_A}$, which as before we denote as $v$.

\paragraph{Showing that $v$ and $u_d$ are roughly orthogonal to each other}
As before we decompose $\bar{G}_n$ as 
\begin{align}
    \bar{G}_n = nvv^\top+ \bar{G}_n - nvv^\top
\end{align}
and note from Weyl's Lemma and definition of $\lambda_i$, that
\begin{align}
    \lambda_i(\bar{G}_n) \leq  \norm{\bar{G}_n - nvv^\top},
\end{align}
for $i \in \{2,3,\cdots, d\}$.
Now decomposing $\norm{\bar{G}_n - nvv^\top}$ into the two sets, one containing arms belonging to the $\mathrm{OPT}_\epsilon(\theta)$ set and another where arms do not belong to $\mathrm{OPT}_\epsilon(\theta)$, and further using a generic upper bound of $\sqrt{\lambda_{\max} (A)}$ for the size of all arms in $\cX$, we have, 
\begin{align}
    \norm{\bar{G}_n - nvv^\top} \leq \sum_{s: A_s \in \mathrm{OPT}_\epsilon(\theta) }\sqrt{\lambda_{\max} (A)}\expect{\norm{A_s - v}} + \sum_{s: A_s \notin \mathrm{OPT}_\epsilon(\theta) }\sqrt{\lambda_{\max} (A)}\expect{\norm{(A_s - v)}}
\end{align}
(See previous section of the spherical set section \ref{subsec:sphere} for details of this decomposition).
For $A_s \in \mathrm{OPT}_\epsilon(\theta)$, we need to find an upper bound for $\norm{A_s - v}$. Now recall the change of variable formula (Equations \ref{eq:El0} and \ref{eq:El1}) from the last section, namely 
\begin{align*}
\mathrm{OPT}_{\cX}(\theta) =  A^{1/2}\mathrm{OPT}_{\cY}(A^{1/2}\theta)
\end{align*}
and
\begin{align*}
    \mathrm{OPT}_{\epsilon, \cX}(\theta) = A^{1/2}\mathrm{OPT}_{\epsilon, \cY}(A^{1/2}\theta)\;,
\end{align*}
where $\cX$ and $\cY$ are the ellipse and unit sphere respectively.

Thus for $A_s \in \mathrm{OPT}_{\epsilon,\cX}(\theta)$, there exists a $x \in \mathrm{OPT}_{\epsilon,\cY}(A^{1/2}\theta)$, such that $A_s = A^{1/2}x$. Therefore, for any $A_s \in \mathrm{OPT}_{\epsilon,\cX}(\theta)$
\begin{align}
    \norm{A_s - v} = \norm{A^{1/2}x - A^{1/2}\mathrm{OPT}_{\cY}(A^{1/2}\theta)} 
\end{align} for some $x \in \mathrm{OPT}_{\epsilon,\cY}(A^{1/2}\theta) $.
Now from the section on the sphere it follows, that for any $x \in \mathrm{OPT}_{\epsilon,\cY}(A^{1/2}\theta)$, we have
\begin{align}
    \norm{x - \mathrm{OPT}_{\cY}(A^{1/2}\theta)} \leq \sqrt{\frac{2\epsilon}{\norm{A^{1/2}\theta}}}
\end{align}
(see the section  on the sphere for the full detail \ref{subsec:sphere}). Therefore for any $A_s \in \mathrm{OPT}_{\epsilon,\cX}(\theta) $ we have
\begin{align*}
    \norm{A_s - v} \leq \norm{A^{1/2}}\sqrt{\frac{2\epsilon}{\norm{A^{1/2}\theta}}}\;.
\end{align*}
Thus we have
\begin{align*}
    \norm{\bar{G}_n - nvv^\top} \leq \sqrt{\lambda_{\max} (A)}\norm{A^{1/2}}\sqrt{\frac{2\epsilon}{\norm{A^{1/2}\theta}}}\mathbb{E}_{\theta}[\mathrm{N}_{\epsilon, n}(\theta)] + \lambda_{\max}(A)\mathbb{E}_{\theta}[n - \mathrm{N}_{\epsilon, n}(\theta)]\;.
\end{align*}
Now using the estimates of $\mathbb{E}_\theta[\mathrm{N}_{\epsilon, n}(\theta)]$ (Equation \ref{eq:estimate}) we have 
\begin{align}
    \norm{\bar{G}_n - nvv^\top} \leq \sqrt{\lambda_{\max} (A)}\norm{A^{1/2}}\sqrt{\frac{2\epsilon}{\norm{A^{1/2}\theta}}}n + \lambda_{\max}(A)\frac{c\sqrt{n}}{\epsilon}\;.
\end{align}
The rest of the proof remains the same as in the spherical case (see subsection \ref{subsec:sphere}), but now $n_0$ would depend upon the matrix $A$ as well.
\paragraph{To show that the sets $\mathrm{OPT}_{\epsilon}(\theta)$ and $\mathrm{OPT}_\epsilon(\theta + O(1/n^{1/4})u_d)$ are disjoint.} 

Let us first rewrite what we need to show. Namely we want to find an $\alpha$ such that 
\begin{align}
    \mathrm{OPT}_{\epsilon,\cX}(\theta) \cap \mathrm{OPT}_{\epsilon, \cX}(\theta + \alpha u_d) = \emptyset\;.
\end{align}
Now using the change of variables (Equations \eqref{eq:El0} and \eqref{eq:El1}) this is equivalent to showing
\begin{align}
    A^{1/2}\mathrm{OPT}_{\epsilon,\cY}(A^{1/2}\theta) \cap A^{1/2}\mathrm{OPT}_{\epsilon, \cY}(A^{1/2}(\theta + \alpha u_d)) = \emptyset\;,
\end{align}
where $\cY$ is the sphere.
By the bijection of $A^{1/2}$, the above is equivalent to showing
\begin{align}
    \mathrm{OPT}_{\epsilon,\cY}(A^{1/2}\theta) \cap \mathrm{OPT}_{\epsilon, \cY}(A^{1/2}\theta + \alpha A^{1/2}u_d) = \emptyset\;.
\end{align}
Now we know $\langle v, u_d \rangle \approx 0$ where $v$ is the optimal arm for the ellipsoid (see the previous paragraph). But for the ellipsoid we know $v = \frac{A\theta}{\norm{\theta}_A}$. Thus $\langle \frac{A\theta}{\norm{\theta}_A}, u_d \rangle \approx 0$. This implies $\langle A^{1/2}\theta, A^{1/2}u_d\rangle \approx 0$. Thus now we can use the result for the spherical action set (Lemma \ref{lemma:step}), namely,
\begin{align}
    \mathrm{OPT}_{\epsilon,\cY}\Big(\frac{A^{1/2}\theta}{\norm{A^{1/2}\theta}}\Big) \cap \mathrm{OPT}_{\epsilon, \cY}\Big(\frac{A^{1/2}\theta}{\norm{A^{1/2}\theta}} + \alpha' \frac{A^{1/2}u_d}{\norm{A^{1/2}u_d}}\Big) = \emptyset
\end{align}
where $\alpha' = O(\frac{1}{n^{1/4}})$. Thus multiplying by $\norm{A^{1/2}\theta}$, we have
\begin{align}
    \mathrm{OPT}_{\epsilon,\cY}\Big(A^{1/2}\theta\Big) \cap \mathrm{OPT}_{\epsilon, \cY}\Big(A^{1/2}\theta + \alpha' \frac{{\norm{A^{1/2}\theta}}}{\norm{A^{1/2}u_d}}A^{1/2}u_d\Big) = \emptyset\;.
\end{align}
Thus the required $\alpha = \alpha' \frac{{\norm{A^{1/2}\theta}}}{\norm{A^{1/2}u_d}}$.

\paragraph{Finishing the proof}
The rest of the proof follows as the ones shown in the previous two sections (paragraph \ref{par:finishing_the_proof}). By ensuring that for a large enough $n$, such that $\norm{\theta'} \leq 2\norm{\theta}$, and using the implication of the dijointedness of the $\epsilon$-optimal sets for $\theta$ and $\theta'$ in the Garivier-Kauffman inequality, we get
\begin{align}
\lambda_d \geq \frac{c_1}{\alpha^2} = \frac{c_1\sqrt{n}{\norm{A^{1/2}u_d}}^2}{\norm{A^{1/2}\theta}^2}  \geq \frac{c_1\sqrt{n}\lambda_{\min}(A)}{\norm{A^{1/2}\theta}^2} = \frac{c_1\sqrt{n}}{\norm{A^{-1/2}}^2\norm{A^{1/2}\theta}^2}\;,
\end{align}
for some positive constant $c_1$. This completes the proof.
\end{proof}

\paragraph{Non centred ellipsoids.}Even though we are considering centred ellipsoids, our proofs readily extend to non-centred ellipsoids.
Consider $\cX = \{x : (x-a)^\top M^{-1} (x-a)\}$ a non-centred ellipsoid. Let $y= x-a$, Then for all such $x$ in $\cX$, $y$ satisfies the equation of the following centred ellipsoid $\cY = \{y : y^\top M^{-1} y = 1\}$. Now with the usual definitions of $\mathrm{OPT}_{\cX}(\theta)$, $\mathrm{OPT}_{\cY}(\theta)$, $\mathrm{OPT}_{\epsilon, \cX}(\theta)$ and $\mathrm{OPT}_{\epsilon, \cY}(\theta)$, observe that the following relations hold true
\begin{gather*}
    \mathrm{OPT}_{\cX}(\theta) = \mathrm{OPT}_{\cY}(\theta) + a \\
    \mathrm{OPT}_{\epsilon, \cX}(\theta) = \mathrm{OPT}_{\epsilon, \cY}(\theta) + a\;.
\end{gather*}
With these relations and using arguments similar to this section (subsection \ref{subsec:Ellipsoidal}), we have our result readily.

\subsection{Locally Constant Hessian (LCH) Action Spaces}
\label{subsec:LCH}
In this section we prove Theorem~\ref{thm:main-result}. We begin by recalling our definition of LCH surfaces (Definition~\ref{def:action}) and Theorem~\ref{thm:main-result}:
\begin{definition}[Locally Constant Hessian (LCH) surface]
\label{def:app-action}
Consider the action space defined by $\mathcal{X} = \{x \in \mathbb{R}^d : f(x) = c\}$, where $f : \mathbb{R}^d \to \mathbb{R}$ is a $C^2$ function (i.e., all second-order partial derivatives of $f$ exist and are continuous) and $c \in \mathbb{R}$. Let $\theta \in \mathbb{R}^d$. $\mathcal{X}$ is said to be a LCH surface w.r.t. $\theta$ if: (i) there is a unique reward-optimal arm with respect to $\theta$ (denoted by  $\mathrm{OPT}_\cX(\theta) = \arg\max_{x \in \cX}\langle x,\theta\rangle$), and (ii) there is an open neighborhood $U \subset \mathbb{R}^d$ of $\mathrm{OPT}_{\cX}(\theta)$ over which the Hessian of $f$ is constant and positive-definite. 
\end{definition}
We are now ready to prove our main result.

\begin{theorem}
\label{thm:app-main-result}
Let the action-space $\cX$ be a Locally Constant Hessian(LCH) surface in $\mathbb{R}^{d-1}$ w.r.t. a bandit parameter $\theta$. Let $\bar{G}_n = \mathbb{E}_\theta\left[\sum_{s=1}^n A_sA_s^\top\right]$, where $A_s$ are arms in $\cX$ drawn according to some bandit algorithm. For any bandit algorithm which suffers expected regret at most $O(\sqrt{n})$, 
\begin{align*}
 \lambda_{\min}(\bar{G}_n) = \Omega(\sqrt{n})\;.   
\end{align*}
That is, there exists an $n_0$ and a constant $\gamma >0$, such that for all $n \geq n_0$, $\lambda_{\min}(\bar{G}_n) \geq \gamma\sqrt{n}$.
\end{theorem}

\begin{remark}
The constant $\gamma$ depends upon the condition number of the Hessian, the algorithmic constants hidden by $O()$, and the size of the bandit parameter $\theta$. The constant $n_0$ depends on the algorithmic constants hidden by $O()$, the size of the neighbourbood over which the Hessian is constant, the size of the action domain $\norm{\cX}$, the singular value of the Hessian and the size of the bandit parameter $\theta$.
\end{remark}
 \vspace{2mm}
 \begin{remark}
 The proof essentially follows the proof of Section~\ref{subsubsec:Ellipoidal}. We present it here for the sake of completeness.
 \end{remark}
\begin{proof}
Consider, the action space given by the curve $\cX = \{x \in \mathbb{R}^d : f(x) = 1\}$. Assume that $\cX$ is upper bounded in norm by a constant represented as $\lambda(\cX)$. Let us define $\mathrm{OPT}_{\cX}(\theta)$ and $\mathrm{OPT}_{\epsilon,\cX}(\theta)$ as before for a bandit parameter $\theta$. Let $U$ be an open-neighbourhood of $\mathrm{OPT}_{\cX}(\theta)$, over which the Hessian of $f$, $\nabla^2 f$, is a constant $H$, and is positive-definite.

We shall show that in this setup the conclusion of $\lambda_{\min}(\bar{G}_n)$ holds true.

As before, we define $\bar{G}_n = \expect{\sum_{s=1}^n A_sA_s^\top}$ where $A_s \in \cX$ and we have a eigenvalue decomposition of $\sum_{i=1}^d \lambda_i u_i u_i^\top$.

We make a Taylor Series approximation of $f$ about $\mathrm{OPT}_{\cX}(\theta)$ (denoted here as $x^*$ for notational ease) in the neighbourhood of $U$.
\begin{align}
    f(x) = f(x^*) +(x-x^*)^\top\nabla f(x^*) + (x-x^*)^\top\nabla^2f(x^*)(x-x^*).
\end{align}
(Note that we have assumed the Hessian is constant in $U$, so there are no third order terms.)
Simplifying the above the expression, we get the following quadratic term,
\begin{align}
    (x-x^*)^\top\nabla^2f(x^*)(x-x^*) + (x-x^*)^\top\nabla f(x^*) = 0\;.
\end{align}
Completing the squares we get the following equation for the ellipsoid, 
\begin{align}
    (x-a)^\top M^{-1} (x-a) = 1
\end{align}
    where $M^{-1} = \frac{H}{4b^\top H^{-1}b}$, $a= x^* - 1/2 H^{-1}b$ and $b= \nabla f(x^*)$, for all $x \in U$.

Thus for all $x \in U \subset \cX$, $x$ satisfies the equation of the ellipse $\cY = \{x : (x-v)^\top M^{-1} (x-v) = 1\}$. The following claims are immediately clear $\mathrm{OPT}_{\cX}(\theta) = \mathrm{OPT}_{\cY}(\theta)$ and for small $\epsilon$ such that $\mathrm{OPT}_{\epsilon, \cX}(\theta) \subset U$, we have $\mathrm{OPT}_{\epsilon, \cX}(\theta) = \mathrm{OPT}_{\epsilon, \cY}(\theta)$.

Let us then choose an $\epsilon$ small enough so that $\mathrm{OPT}_{\epsilon, \cX}(\theta) = \mathrm{OPT}_{\epsilon, \cY}(\theta)$.

As a first step we see that the estimates of $\mathbb{E}_{\theta'}[\mathrm{N}_{\epsilon, n}(\theta')] \geq n - \frac{c\sqrt{n}}{\epsilon}$ remain true, for the entire action space $\cX$, for all $\theta'$ such that $\norm{\theta'} \leq 2\norm{\theta}$ for which $\expect{R_n(\theta')} \leq c\sqrt{n}$ for some positive constant $c$. 

As the next order of business we need to show $\mathrm{OPT}_{\cX}(\theta)$ is approximately orthogonal to $u_d$. For this first we note that $\mathrm{OPT}_{\cX}(\theta) = \mathrm{OPT}_{\cY}(\theta)$. Thus it suffices to show that $\mathrm{OPT}_{\cY}(\theta)$ is approximately orthogonal to $u_d$.

As before denote $\mathrm{OPT}_{\cY}(\theta) = v$ and decompose $\bar{G}_n$ as 
\begin{align}
    \bar{G}_n = nvv^\top+ \bar{G}_n - nvv^\top
\end{align}
and note from Weyl's Lemma and definition of $\lambda_i$, that
\begin{align}
    \lambda_i(\bar{G}_n) \leq  \norm{\bar{G}_n - nvv^\top},
\end{align}
for $i \in \{2,3,\cdots, d\}$.
Now decomposing $\norm{\bar{G}_n - nvv^\top}$ into the two sets, one containing arms belonging to the $\mathrm{OPT}_{\epsilon, \cX}(\theta)$ set and another where arms do not belong to the $\mathrm{OPT}_{\epsilon, \cX}(\theta)$ set and further using the generic upper bound of $\lambda_{\max} (\cX)$ for the size of all arms we have, 
\begin{align}
    \norm{\bar{G}_n - nvv^\top} \leq \sum_{s: A_s \in \mathrm{OPT}_\epsilon(\theta) }\lambda(\cX)\expect{\norm{A_s - v}} + \sum_{s: A_s \notin \mathrm{OPT}_\epsilon(\theta) }\lambda(\cX)\expect{\norm{(A_s - v)}}\;.
\end{align}
Note as before, this can be be upper bounded using the estimates of $\mathbb{E}_{\theta}[\mathrm{N}_{\epsilon, n}(\theta)]$ as 
\begin{align}
   \norm{\bar{G}_n - nvv^\top} \leq \frac{c\sqrt{n}\lambda(\cX)^2}{\epsilon} + n\lambda(\cX)\sup_{A_s \in \mathrm{OPT}_{\epsilon, \cX}(\theta)}\norm{A_s - v}\;.
\end{align}
Now for small $\epsilon$ such that $\mathrm{OPT}_{\epsilon, \cX}(\theta) = \mathrm{OPT}_{\epsilon, \cY}(\theta)$ we have $\sup_{A_s \in \mathrm{OPT}_{\epsilon, \cX}(\theta)}\norm{A_s - v} = \sup_{A_s \in \mathrm{OPT}_{\epsilon, \cY}(\theta)}\norm{A_s - v}$ for which we know $\norm{A_s - v} \leq O(\sqrt{\epsilon})$. Thus by choosing $\epsilon$ as $O(1/\sqrt{n})$ for sufficiently large $n$, such that $\mathrm{OPT}_{\epsilon, \cX}(\theta) = \mathrm{OPT}_{\epsilon, \cY}(\theta)$, we have (by choosing appropriate constant for reference see the section on the sphere)
\begin{align}
    \norm{\bar{G}_n - nvv^\top} \leq 0.1n\;.
\end{align}
Thus we have the separation of eigen-values between $\lambda_1(nvv^\top)$ and $\lambda_2(\bar{G}_n)$ as at least $0.9n$. This gives us from the Davis-Kahan Theorem
\begin{align}
     |v^\top u_d| \leq 1/9
\end{align}
(See the previous sections for details). Thus we have $\mathrm{OPT}_\cX(\theta)$ is approximately orthogonal to $u_d$.

As the next order of business we need to find a perturbed version of $\theta$, namely $\theta' = \theta + \alpha u_d$ such that $\mathrm{OPT}_{\epsilon, \cX}(\theta)$ and $\mathrm{OPT}_{\epsilon, \cX}(\theta')$ are disjoint. First let us choose $\epsilon$ small enough by sufficiently large $n$ such $\mathrm{OPT}_{\epsilon, \cX}(\theta) = \mathrm{OPT}_{\epsilon, \cY}(\theta)$. Now as before for the ellipsoid case we have $\mathrm{OPT}_{\epsilon, \cY}(\theta) \cap \mathrm{OPT}_{\epsilon, \cY}(\theta + O(1/n^{1/4})u_d) = \emptyset $. Now by the continuity of $\mathrm{OPT}_\cY(\theta) = \frac{M\theta}{\norm{\theta}_M}+a$, we have for sufficiently large $n$, $\mathrm{OPT}_\cY(\theta+O(1/n^{1/2})u_d) \in U$ and hence for small $\epsilon$ we have $\mathrm{OPT}_{\epsilon, \cY}(\theta + O(1/n^{1/4})u_d) \subset U$. Thus for large enough $n$, we have two disjoint sets $\mathrm{OPT}_{\epsilon, \cY}(\theta)$ and $\mathrm{OPT}_{\epsilon, \cY}(\theta + O(1/n^{1/4})u_d)$ both contained in $U$ and thus
$\mathrm{OPT}_{\epsilon, \cX}(\theta) \cap \mathrm{OPT}_{\epsilon, \cX}(\theta + O(1/n^{1/4})u_d) = \emptyset$.
As before, we can now use the Gariver-Kauffman inequality to get
\begin{align}
    \lambda_d(\bar{G}_n) \geq \gamma\sqrt{n}
\end{align}
for large enough $n$ depending upon $U$, and some constant $\gamma$ depending upon the Hessian. This completes our proof.
\end{proof}

%% file: app_convex.tex
\section{LOCALLY CONVEX SURFACES (PROOF OF THEOREM~\ref{thm:locally_convex})}
\label{sec:app-convex}
In this section we provide a proof of Theorem~\ref{thm:locally_convex}. We begin again by by recalling or definition of Locally Convex surfaces (Definition ~\ref{def:locally_convex_action}) and Theorem~\ref{thm:locally_convex} for Locally Convex Action Sets.
\begin{definition}[Locally Convex surface]
\label{def:app-locally_convex_action}
Consider an action space $\mathcal{X} = \{x \in \mathbb{R}^d : f(x) = c\}$, where $f : \mathbb{R}^d \to \mathbb{R}$ is a $C^2$ function (i.e., all second order partial derivatives exist and are continuous). With $\mathrm{OPT}_{\mathcal{X}}(\theta)$ being the optimal arm defined as before, let the Hessian of $f$ at $\mathrm{OPT}_{\mathcal{X}}(\theta)$, denoted as $\nabla^2f(\mathrm{OPT}_{\mathcal{X}}(\theta))$, be positive definite. Then, $\mathcal{X}$
is said to be a Locally Convex surface.
\end{definition}

\begin{theorem}
\label{thm:app-locally_convex}
Let $\cX$ be a Locally Convex action space and $\Bar{G}_n$, be  the expected design matrix. For any bandit algorithm which suffers expected regret at most $O(\sqrt{n})$, there exists a real number $s$ in the half-open interval $(0,\frac{1}{2}]$ such that 
\begin{equation*}
    \lambda_{\min}(\Bar{G}_n) = \Omega(n^s)~.
\end{equation*}
\end{theorem}

\begin{remark}
The exponent $s$ defined in Theorem \ref{thm:app-locally_convex} depends explicitly on the geometry of the action-space. In general it depends on how well the surface approximates a LCH surface. The best we can say when no other information is given about the action space is that the eigenvalue grows polynomially, which still results in a polynomial rate of the parameter estimation.
\end{remark}

\begin{proof}[Proof of Theorem \ref{thm:app-locally_convex}]
The proof utilizes the following lemmas,

\begin{lemma}
\label{lemma:quad_approx}
We can construct a quadratic approximation $\cE$ about $\mathrm{OPT}_\cX(\theta)$ similar to the construction for the LCH case.
\end{lemma}

\begin{lemma}
\label{lemma:injective}
There exists a one-to-one map, $\Bar{\mathbf{P}}_{\cX,\cE} : \cX \to \cE$ and a positive real number $p$ in the interval $(0,1]$, such that the $\epsilon$-optimal set in $\cX$, $\mathrm{OPT}_{\epsilon,\cX}(\theta)$ is contained in an $O(\epsilon^p)$-optimal set in $\cE$ for small enough $\epsilon$ .
\end{lemma}

\begin{lemma}
\label{lemma:construction}
There exists a positive real number $s$ in the interval $(0, \frac{1}{2}]$, such that for $\theta' = \theta + O(\sqrt{\frac{1}{n^s}})u_d$, the $O(\epsilon^{2s})$-optimal sets in the approximation $\cE$, $\mathrm{OPT}_{O(\epsilon^{2s}),\cE}(\theta)$ and $\mathrm{OPT}_{O(\epsilon^{2s}),\cE}(\theta')$ are disjoint, and the image of the $\epsilon$-optimal set $\mathrm{OPT}_{\epsilon,\cX}(\theta')$ under the operator defined in Lemma \ref{lemma:injective} is contained in $\mathrm{OPT}_{O(\epsilon^{2s}),\cE}(\theta')$.
\end{lemma}

Thus using Lemma \ref{lemma:injective} and Lemma \ref{lemma:construction}, we ensure that the $\epsilon$-optimal sets of $\cX$ for the bandit parameters $\theta$ and $\theta'$ are disjoint. The remainder of the proof is the same as in the earlier cases. 
\end{proof}

\begin{proof}[Proof of Lemma \ref{lemma:quad_approx}]
By the $C^2$ definition, there exists an open set $W \in \mathbb{R}^d$ such that the Hessian of $f$ over $W \cap \cX \triangleq U$, $\nabla^2(f)(U)$ exists and is positive definite. Let $\cE$ be a quadratic approximation ellipsoid (similar to the construction as in the locally constant hessian (LCH) case) about $\mathrm{OPT}_\cX(\theta)$.
\end{proof}

\begin{proof}[Proof of Lemma \ref{lemma:injective}]
Let $\Bar{\mathbf{P}}_{\cX,\cE} : \cX \to \cE$ be the best approximation operator denoted by 
\begin{equation*}
    \Bar{\mathbf{P}}_{\cX,\cE}(y) = \argmin_{x \in \cE}\norm{x-y}
\end{equation*}
for any $y$ in $\cX$. This operator finds the nearest point to a point in the action set $\cX$ to the ellipsoid approximation $\cE$.

From continuity of $f$ and construction of $\cE$, there exists a small neighbourhood $W$ about $\mathrm{OPT}_\cX(\theta)$ such that the $\Bar{\mathbf{P}}_{\cX,\cE}$ operator is injective. As in the proof of the Locally Constant Hessian (LCH) case, we shall restrict our attention to the neighbourhood of $U \cap W \triangleq V$.

Let $x \in \mathrm{OPT}_{\epsilon,\cX}(\theta)$. By definition,
\begin{equation*}
    x^\top\theta \geq \mathrm{OPT}_\cX(\theta)^\top\theta - \epsilon.
\end{equation*}
Now $x = \Bar{\mathbf{P}}_{\cX,\cE}(x) + r \Vec{u}$ where $\Vec{u}$ is the normal at $x$ and $r$ is $\norm{\Bar{\mathbf{P}}_{\cX,\cE}(x) - x}$. Note that as $\epsilon \to 0$, $\Bar{\mathbf{P}}_{\cX,\cE}(x) \to \mathrm{OPT}_\cX(\theta)$ and $x \to \mathrm{OPT}_\cX(\theta) $. Therefore $r \to 0$. Thus $r = O(\epsilon^p)$, for some $p > 0$. Then, without loss of generality assuming $\norm{\theta} = 1$, we have by Cauchy-Schwartz, 
\begin{equation*}
    \Bar{\mathbf{P}}_{\cX,\cE}(x)^\top\theta + r \geq \Bar{\mathbf{P}}_{\cX,\cE}(x)^\top\theta + r \Vec{u}^\top\theta = x^\top\theta \geq \mathrm{OPT}_\cX(\theta)^\top\theta - \epsilon = \mathrm{OPT}_\cE(\theta)^\top\theta - \epsilon\;.
\end{equation*}
Thus $\Bar{\mathbf{P}}_{\cX,\cE}(x) \in \mathrm{OPT}_{O(\epsilon^{\min{(1,p)}}),\cE}(\theta)$. 

Let $\epsilon_1$ be such that for all $\epsilon \leq \epsilon_1$, $\mathrm{OPT}_{\epsilon,\cX}(\theta) \subset V$, then we have $\Bar{\mathbf{P}}_{\cX,\cE}(\mathrm{OPT}_{\epsilon,\cX}(\theta)) \subset \mathrm{OPT}_{O(\epsilon^{\min{(1,p)}}),\cE}(\theta)$ for all $\epsilon \leq \epsilon_1$.
\end{proof}

\begin{proof}[Proof of Lemma \ref{lemma:construction}]
From the ellipsoidal case, we have by construction, $\theta' = \theta + \alpha u_d$, where $\alpha = O(\sqrt{\epsilon})$ and $u_d$ is the eigen-vector corresponding to the minimum eigen-value to ensure separation. This required showing that $u_d$ and $\mathrm{OPT}(\theta)$ are roughly perpendicular to each other. For the current case, it is also possible to show that $u_d$ and $\mathrm{OPT}_\cX(\theta)$ is approximately perpendicular under the polynomial approximation assumption for $\epsilon = O(\frac{1}{\sqrt{n}})$. We leave this as an easy exercise.

Thus let us construct $\theta' = \theta + \alpha u_d$, where $\alpha$, would have to be carefully constructed, such that as $\epsilon \to 0$, $\alpha \to 0$. Let, $x \in \mathrm{OPT}_{\epsilon,\cX}(\theta')$. By definition, 
\begin{equation*}
    x^\top\theta' \geq \mathrm{OPT}_\cX(\theta')^\top\theta' - \epsilon.
\end{equation*}
For $\alpha$ very small, we have, $\norm{\theta'} \approx 1$, and as before using the approximation operator and Cauchy-Schwartz, we have
\begin{equation}
    \Bar{\mathbf{P}}_{\cX,\cE}(x)^\top\theta' + r_1 \geq \Bar{\mathbf{P}}_{\cX,\cE}(\mathrm{OPT}_\cX(\theta'))^\top \theta' + r_2 \Vec{u}^\top \theta' - \epsilon\;.
\end{equation}
where $r_1$ is the approximation error $\norm{\Bar{\mathbf{P}}_{\cX,\cE}(x) - x}$, $r_2$ is the approximation error $\norm{\Bar{\mathbf{P}}_{\cX,\cE}(\mathrm{OPT}_\cX(\theta')) - \mathrm{OPT}_\cX(\theta')}$ and $\Vec{u}$ is the normal at $\mathrm{OPT}_\cX(\theta')$.

As $\epsilon \to 0$, we have $\theta' \to \theta$ (because $\alpha \to 0$) and therefore $x \to \mathrm{OPT}_\cX(\theta')$. Thus $r_1 = O(\epsilon^q)$ for some $q>0$ by the polynomial approximation. Also $\Vec{u}$ approaches the normal at $\mathrm{OPT}_\cX(\theta)$ and thus by continuity of $f$, there exists an $\epsilon_2$ such that for all $\epsilon \leq \epsilon_2$ we have $\Vec{u}^\top \theta' \geq 0$. Thus for small enough $\epsilon$, we have for any $x \in \mathrm{OPT}_{\epsilon,\cX}(\theta')$, 
\begin{equation}
    \Bar{\mathbf{P}}_{\cX,\cE}(x)^\top\theta' + O(\epsilon^q) \geq \Bar{\mathbf{P}}_{\cX,\cE}(\mathrm{OPT}_\cX(\theta'))^\top \theta' - \epsilon\;.
\end{equation}
Note that as $\epsilon \to 0$, we have $\Bar{\mathbf{P}}_{\cX,\cE}(\mathrm{OPT}_\cX(\theta')) \to \mathrm{OPT}_\cE(\theta)$ and $\mathrm{OPT}_\cE(\theta') \to \mathrm{OPT}_\cE(\theta)$, and thus we have $\norm{\Bar{\mathbf{P}}_{\cX,\cE}(\mathrm{OPT}_\cX(\theta')) - \mathrm{OPT}_\cE(\theta')} = O(\epsilon^r)$ for some $r >0$. This implies, 
\begin{equation*}
    \mathrm{OPT}_\cE(\theta')^\top \theta' - \Bar{\mathbf{P}}_{\cX,\cE}(\mathrm{OPT}_\cX(\theta'))^\top \theta' \leq \norm{\mathrm{OPT}_\cE(\theta') - \Bar{\mathbf{P}}_{\cX,\cE}(\mathrm{OPT}_\cX(\theta'))}\norm{\theta'} \leq O(\epsilon^r)\;.
\end{equation*}
Using this relation we get, 
\begin{equation*}
    \Bar{\mathbf{P}}_{\cX,\cE}(x)^\top\theta' \geq \mathrm{OPT}_\cE(\theta')^\top \theta' - \epsilon - O(\epsilon^r) - O(\epsilon^q). 
\end{equation*}

Thus we have $\Bar{\mathbf{P}}_{\cX,\cE}(x) \in \mathrm{OPT}_{O(\epsilon^{\min{(1,q,r)}}),\cE}(\theta')$. Let $\epsilon_3$ be such that for all $\epsilon \leq \epsilon_3$, $\mathrm{OPT}_{\epsilon,\cX}(\theta') \subset V$, then we have $\Bar{\mathbf{P}}_{\cX,\cE}(\mathrm{OPT}_{\epsilon,\cX}(\theta')) \subset \mathrm{OPT}_{O(\epsilon^{\min{(1,q,r)}}),\cE}(\theta)$ for all $\epsilon \leq \epsilon_3$.

Thus for all $\epsilon \leq \min{(\epsilon_1,\epsilon_2,\epsilon_3)}$, we have 
$\Bar{\mathbf{P}}_{\cX,\cE}(\mathrm{OPT}_{\epsilon,\cX}(\theta)) \subset \mathrm{OPT}_{O(\epsilon^{\min{(1,p)}}),\cE}(\theta) \subset \mathrm{OPT}_{O(\epsilon^{\min{(1,p,q,r)}}),\cE}(\theta)$ 

and $\Bar{\mathbf{P}}_{\cX,\cE}(\mathrm{OPT}_{\epsilon,\cX}(\theta')) \subset \mathrm{OPT}_{O(\epsilon^{\min{(1,q,r)}}),\cE}(\theta') \subset \mathrm{OPT}_{O(\epsilon^{\min{(1,p,q,r)}}),\cE}(\theta')$.

Now from the LCH case, we have $\mathrm{OPT}_{O(\epsilon^{\min{(1,p,q,r)}}),\cE}(\theta)$ and $\mathrm{OPT}_{O(\epsilon^{\min{(1,p,q,r)}}),\cE}(\theta')$ are disjoint if $\theta' = \theta + O(\sqrt{\epsilon^{\min{(1,p,q,r)}}}) u_d$, and by construction based on the injectivity of the approximation operator, $\mathrm{OPT}_{\epsilon,\cX}(\theta)$ and $\mathrm{OPT}_{\epsilon,\cX}(\theta')$ are as well disjoint.

This gives us by choice of $\epsilon = O(\frac{1}{\sqrt{n}})$ 
\begin{align*}
    \lambda_d \geq \Omega(n^s)
\end{align*}
where $s \in (0,1/2]$ and defined as $s = \frac{1}{2}\min{(1,p,q,r)}$. This completes our proof.
\end{proof}

\begin{remark}
For the LCH case, we have no approximation error and the exponents defined by $p, q, r$, are strictly more than $1$ (they are in fact $\infty$) and we get back the results of Theorem~\ref{thm:app-main-result} of $\Omega(\sqrt{n})$.
\end{remark}
\vspace{2mm}
\begin{remark}
Note that the $s$ defined in Theorem~\ref{thm:app-locally_convex} is defined as $\frac{1}{2}\min{(1,p,q,r)}$, where each $p, q, r$ represent the approximation error with respect to an LCH surface. Specifically with $\cE$ being defined as the approximating Ellipsoid and $\cX$ denoting the original action space, we have $p$ defined as the rate at which $\norm{\Bar{\mathbf{P}}_{\cX,\cE}(x) - x}$ goes to $0$, that is $\norm{\Bar{\mathbf{P}}_{\cX,\cE}(x) - x} = O(\epsilon^p)$ for any $x \in \mathrm{OPT}_{\epsilon, \cX}(\theta)$. Similarly $q$ is defined as $\norm{\Bar{\mathbf{P}}_{\cX,\cE}(x) - x} = O(\epsilon^q)$ for any $x \in \mathrm{OPT}_{\epsilon, \cX}(\theta')$ and $r$ is defined as $\norm{\Bar{\mathbf{P}}_{\cX,\cE}(\mathrm{OPT}_\cX(\theta')) - \mathrm{OPT}_\cE(\theta')} = O(\epsilon^r)$. 
\end{remark}

\subsection{Example of Locally Convex action space:} In this subsection we demonstrate experimentally an action space which is \emph{convex} and for which the minimum eigenvalue grows slower than $\Omega(\sqrt{n})$. Consider the action-space $\cX = \{x\in \mathbb{R}^d : \norm{x}_{10} \leq 1\}$. Clearly, this is a convex set. We set the bandit parameter as $\theta = (1,\ldots, 1)$, as a vector of $1$s of size $d$. We use Thompson Sampling as the representative algorithm and plot the growth of the minimum eigenvalue versus rounds $n$ as in Section~\ref{sec:experiments}. In order to demonstrate the high probability phenomenon, we form a high confidence band of the mean observation of $\log{\lambda_{\min}(V_n)}/\log{n}$ with three standard deviations of width.We observe the the minimum eigenvalue grows less than $\Omega(\sqrt{n})$.

\begin{figure}[tbhp]
    \centering
    \includegraphics[width = 0.5\linewidth]{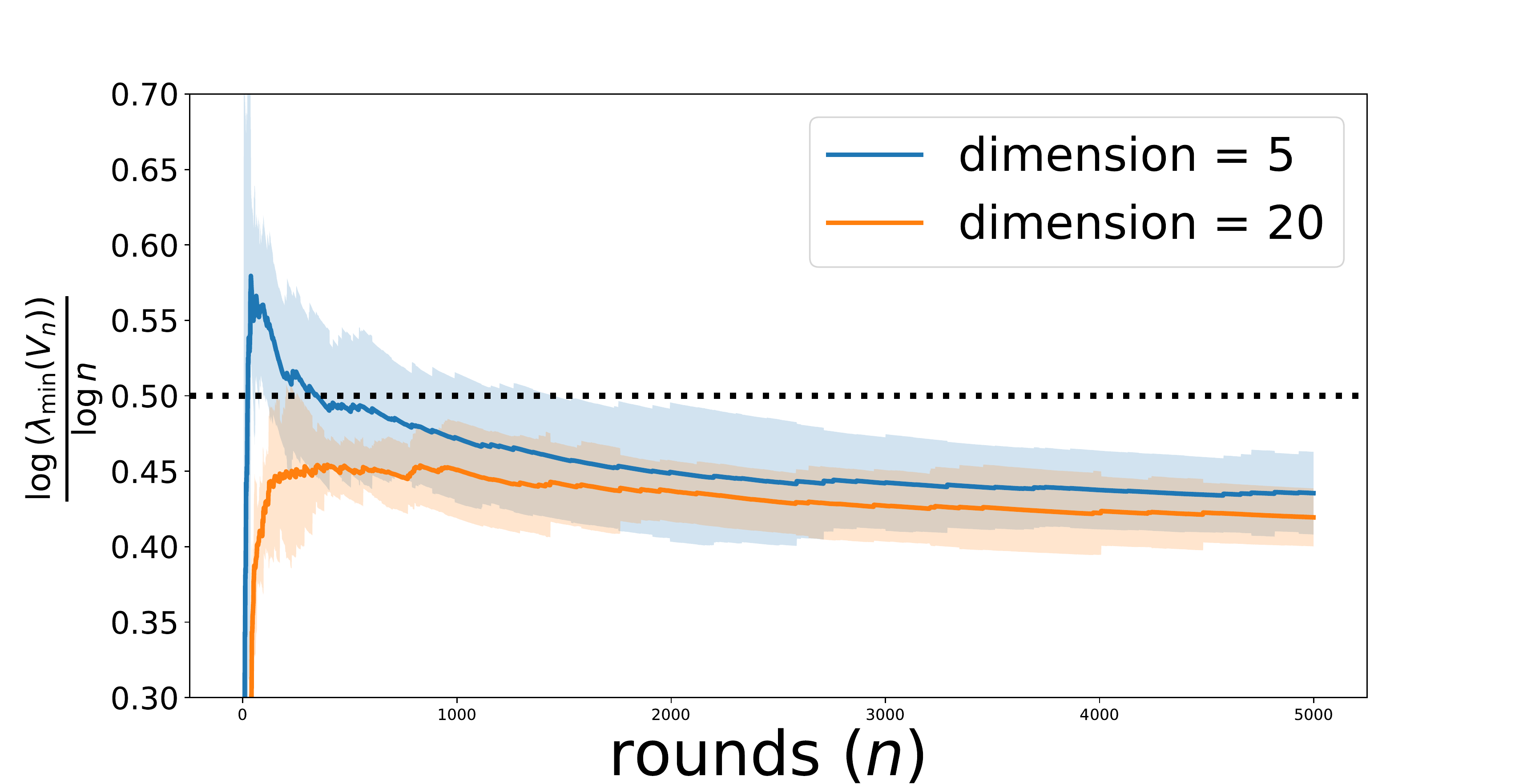}
    \caption{Scaling of the minimum eigenvalue of design matrix (generated by the Thompson sampling algorithm) with time for the action space $\cX = \{x\in \mathbb{R}^d : \norm{x}_{10} \leq 1\}$ . The plots represent averages over $20$ independent trials. The $X$ axis denotes the number of rounds $n$ and the $Y$-axis denotes $\frac{\log{\lambda_{\min}(V_n)}}{\log{n}}$. We plot the mean trend of $\frac{\log{\lambda_{\min}(V_n)}}{\log{n}}$ along with three standard deviations.  The dotted black line is the constant (exponent) $1/2$. Note that $\frac{\log{\lambda_{\min}(V_n)}}{\log{n}}$ settles at value below $1/2$. We plot this for two dimensions $5$ and $20$ to note the dependence on dimension $d$. We note that these corroborate with our theory. }
    \label{fig:my_label}
\end{figure}

\paragraph{Justification : } Let us first observe that the Hessian at $\mathrm{OPT}_{\cX}(\theta)$ is positive definite and continuous for our choice of $\theta$. We can also calculate that for our choice of the action space $\cX$, any point $x \in \mathrm{OPT}_{\epsilon, \cX}(\theta)$,
\begin{align*}
\norm{x - \mathrm{OPT}_{\cX}(\theta)} =\epsilon^{1/10}~.    
\end{align*}
(This is easy to see for dimension $2$.) Thus the $\norm{\Bar{\mathbf{P}}_{\cX,\cE}(x) - x}$ as defined in the proof of Theorem~\ref{thm:app-locally_convex} (see Section~\ref{sec:app-convex}) is of the order of $\epsilon^{1/10}$ and hence $p$ as defined above is $1/10$. This means that $\lambda_{\min}(V_n) = \Omega(n^{0.1})$. This is true as observed from our experiments as observed in Figure~\ref{fig:my_label}. However we specifically also show that the minimum eigenvalue $\lambda_{\min}(V_n)$ is growing at a rate less than $\sqrt{n}$. This corroborates our results.

%% file: app_model_selection.tex
\section{MODEL SELECTION: PROOF OF RESULTS}

\begin{proof}[Proof of Lemma~\ref{lem:sequence}]
We consider doubling epochs, with lengths $n_i = 2^{i-1} n_1$, where $n_1$ is the initial epoch length and $N$ is the total number of epochs. Then, from the doubling principle, we get
\begin{align*}
    \sum_{i=1}^N 2^{i-1}n_1 = n \,\, \Rightarrow \, N  = \log_2 \left(1+ n/n_1 \right) =\cO\left(\log (n/n_1)\right).
\end{align*}
Now, consider the $i$-th epoch. Let $\widehat{\theta}_{n_i}$ be the least square estimate of $\theta$ at the end of epoch $i$. The confidence interval at the end of epoch $i$, i.e., after \texttt{OFUL} is run with a norm estimate $b^{(i)}$ for $n_i$ rounds with confidence level $\delta_i$, is given by
\begin{align*}
    \mathcal{B}_{n_i} = \left \lbrace \theta \in \Real^d: \|\theta - \widehat{\theta}_{n_i}\|_{(V_{n_i}+ \lambda I)} \leq \sqrt{\beta_{n_i}(\delta_i)} \right \rbrace.
\end{align*}
Here $\beta_{n_i}(\delta_i)$ denotes the radius and $\Sigma_{n_i}$ denotes the shape of the ellipsoid.
Under Assumption~\ref{ass:high-prob}, the ellipsoid takes the form 
\begin{align*}
    \mathcal{B}_{n_i} = \left \lbrace \theta \in \Real^d: \|\theta - \widehat{\theta}_{n_i}\| \leq \frac{ \sqrt{\beta_{n_i}(\delta_i)}}{\sqrt{\lambda+\gamma_0  n_i}} \right \rbrace,
\end{align*}
with probability at least $1-\delta_i = 1-\delta/2^{i-1}$. Here, we use the fact that $\gamma_{\min}(V_{n_i}) \geq \lambda+ \gamma_0 \sqrt{n_i}$ under Assumption~\ref{ass:high-prob}, provided $n_i \geq n_0$. To ensure this, we choose $n_1 \geq n_0 $. Now, note that $\theta \in \mathcal{B}_{n_i}$ with probability at least $1-\delta_i$. Therefore, we have
\begin{align*}
    \|\widehat{\theta}_{n_i}\| \leq \|\theta\| + \frac{ \sqrt{\beta_{n_i}(\delta_i)}}{\sqrt{\lambda+\gamma_0 \sqrt{n_i}}}
\end{align*}
with probability at least $1-2\delta_i$. Recall that at the end of the $i$-th epoch, \texttt{ALB} set the estimate of $\|\theta\|$ to $b^{(i+1)} = \max_{\theta \in \mathcal{B}_{n_i}} \|\theta\|$. Then,
from the definition of $\mathcal{B}_{n_i}$, we obtain
\begin{align*}
    b^{(i+1)} =\|\widehat{\theta}_{n_i}\| +  \frac{ \sqrt{\beta_{n_i}(\delta_i)}}{\sqrt{\lambda+\gamma_0 \sqrt{n_i}}} 
     \leq \|\theta\| + 2 \frac{ \sqrt{\beta_{n_i}(\delta_i)}}{\sqrt{\lambda+\gamma_0 \sqrt{n_i}}}
\end{align*}
with probability higher than $1-2\delta_i$. If noise is distributed as $\cN(0,1)$, the confidence radius reads
\begin{align*}
    \sqrt{\beta_{n_i}(\delta_i)} = b^{(i)}\sqrt{\lambda}+\sqrt{2\log(1/\delta_i)+d\log(1+n_i/(\lambda d))}~,
\end{align*}
We now substitute $n_i = 2^{i-1}n_1$ and $\delta_i = \frac{\delta}{2^{i-1}}$ to obtain $\sqrt{\lambda+\gamma_0 \sqrt{n_i}} \geq 2^{\frac{i-1}{4}} \sqrt{\gamma_0 \sqrt{n_1}}$, and
\begin{align*}
    \frac{\sqrt{\beta_{n_i}(\delta_i)}}{ \sqrt{\lambda+\gamma_0n_i}} \leq C_1\frac{ b^{(i)}\sqrt{\lambda}}{2^{\frac{i-1}{4}}\sqrt{\gamma_0 \sqrt{n_1}}} + C_2 \frac{i}{2^{\frac{i-1}{4}}\sqrt{\gamma_0\sqrt{n_1}}} \sqrt{2\log(1/\delta_1)+d\log(1+n_1/(\lambda d))}
\end{align*}
for some universal constants $C_1,C_2$.
Using this, we further obtain
\begin{align*}
    b^{(i+1)} &\leq \|\theta\| + 2C_1\frac{ b^{(i)}\sqrt{\lambda}}{2^{\frac{i-1}{4}}\sqrt{\gamma_0\sqrt{n_1}}} + 2C_2 \frac{i}{2^{\frac{i-1}{4}}\sqrt{\gamma_0 \sqrt{n_1}}} \sqrt{2\log(1/\delta_1)+d\log(1+n_1/(\lambda d))} \\
    & \leq \|\theta\| + b^{(i)}\frac{p}{2^{\frac{i-1}{4}}}\sqrt{\frac{1}{\sqrt{n_1}}} + \frac{iq}{2^{\frac{i-1}{4}}} \sqrt{\frac{d}{\sqrt{n_1}}},
\end{align*}
with probability at least $1-2\delta_i$, where we introduce the terms
\begin{align*}
    p = \frac{2C_1\sqrt{\lambda}}{\sqrt{\gamma_0}} \,\,\,\,\text{and} \,\,\, q = \frac{2C_2}{\sqrt{\gamma_0}} \sqrt{2\log(1/\delta_1)+\log(1+n_1/\lambda)}.
\end{align*}
Therefore, with probability at least $1-2\delta_i$, we obtain
\begin{align*}
    b^{(i+1)} - b^{(i)} \leq \|\theta\| - \left(1-\frac{p}{2^{\frac{i-1}{4}}}\frac{1}{n_1^{1/4}}\right)b^{(i)} + \frac{iq}{2^{\frac{i-1}{4}}} \frac{\sqrt{d}}{n_1^{1/4}}.
\end{align*}
Note that by construction, $b^{(i)} \geq \|\theta\|$. Hence, provided $n_1 > \frac{2p^4}{2^i}$, we have
\begin{align*}
    b^{(i+1)} - b^{(i)} \leq \frac{p}{2^{\frac{i-1}{4}}}\frac{\norm{\theta}}{n_1^{1/4}}
     + \frac{iq}{2^{\frac{i-1}{4}}} \frac{\sqrt{d}}{n_1^{1/4}},
\end{align*}
with probability at least $1-2\delta_i$.
From the above expression, we have $\sup_i b^{(i)} < \infty$
with probability greater than or equal to $ 1-\sum_i 2\delta_i = 1-\sum_i 2\delta/2^{i-1} = 1-4\delta$. From the expression of $b^{(i+1)}$ and using the above fact, we get $\lim_{i \rightarrow \infty} b^{(i)} \leq \|\theta\|$.
However, by construction $b^{(i)} \geq \|\theta\|$. Using this, along with the above observation, we obtain
\begin{align*}
    \lim_{i \rightarrow \infty} b^{(i)} = \|\theta\|.
\end{align*}
with probability exceeding $1-4\delta$. Therefore, we deduce that the sequence $\{b^{(1)},b^{(2)},...\}$ converges to $\|\theta\|$ with probability at least $1-4\delta$, and hence our successive refinement algorithm is consistent.

\textbf{Rate of Convergence.} Since $ b^{(i+1)} - b^{(i)} = \Tilde{\mathcal{O}}\left( \frac{i}{2^{i/4}} \right)$
with high probability, the rate of convergence of the sequence $\{b^{(i)}\}_{i=1}^\infty$ is exponential in the number of epochs.

\textbf{A uniform upper bound on $b^{(i)}$.}
Consider the sequences $ \Bigg \{\frac{i}{2^{\frac{i-1}{4}}} \Bigg \}_{i=1}^\infty$ and $ \Bigg \{\frac{1}{2^{\frac{i-1}{4}}} \Bigg \}_{i=1}^\infty$. Let $t_i$ and $u_i$ denote the $i$-th term of the sequences respectively. It is easy to see that $\sup_{i}t_i < \infty $ and $\sup_{i}u_i < \infty$, and that the sequences $\{t_i\}_{i=1}^\infty$ and $\{u_i\}_{i=1}^\infty$ are convergent. Now, we have
\begin{align*}
    b^{(2)} \leq \|\theta\| + u_1 \frac{p b^{(1)}}{n_1^{1/4}} + t_1 \frac{q \sqrt{d}}{n_1^{1/4}}
\end{align*}
with probability at least $1-2\delta$. Similarly, we write $b^{(3)}$ as
\begin{align*}
    b^{(3)} \leq \|\theta\| + u_2 \frac{p b^{(2)}}{n_1^{1/4}} + t_2 \frac{q \sqrt{d}}{n_1^{1/4}} \leq  \left (1+u_2 \frac{p}{n_1^{1/4}} \right) \|\theta\| + \left( u_1 u_2 \frac{p}{n_1^{1/4}} \frac{p}{n_1^{1/4}} b^{(1)} \right) + \left( t_1 u_2 \frac{p}{n_1^{1/4}} \frac{q \sqrt{d}}{n_1^{1/4}} + t_2 \frac{q \sqrt{d}}{n_1^{1/4}}\right)
\end{align*}
with probability at least $1 - 2\delta -\delta = 1-3\delta$. Similarly, we write expressions for $b^{(4)},b^{(5)},...$. Now, provided $n_1 \geq C \,d^2 \left ( \max\{p,q\}  \, b^{(1)} \right)^4 $, where $C $ is a sufficiently large constant, $b^{(i)}$ can be upper-bounded, with with probability at least $1-\sum_{i} 2 \delta_i=1-4\delta$, as
\begin{align*}
    b^{(i)} \leq c_1 \| \theta \| + c_2
\end{align*}
for all $i$, where $c_1,c_2 > 0$ are some universal constants, which are obtained from summing an infinite geometric series with decaying step size.
\end{proof}


\begin{proof}[Proof of Corollary~\ref{thm:norm}]
The cumulative regret of \texttt{ALB} is given by
\begin{align*}
    R(n) \leq \sum_{i=1}^N R^{\texttt{OFUL}}(n_i,\delta_i,b^{(i)}),
\end{align*}
where $N$ denotes the total number of epochs and $R^{\texttt{OFUL}}(n_i,\delta_i,b^{(i)})$ denotes the cumulative regret of \texttt{OFUL}, when it is run with confidence level $\delta_i$ and norm upper bound $b^{(i)}$ for $n_i$ episodes. Using the result of \citet{oful}, we have
\begin{align*}
    R^{\texttt{OFUL}}(n_i,\delta_i,b^{(i)}) =\mathcal{O}\left(b_i\sqrt{dn_i\log n_i} + d\sqrt{n_i\log n_i\log(n_i/\delta)}\right)
\end{align*}
with probability at least $1-\delta_i$. Now, using Lemma~\ref{lem:sequence}, we obtain
\begin{align*}
    R(n) \leq (c_1 \| \theta \| + c_2) \sum_{i=1}^N \mathcal{O}\left( \sqrt{dn_i\log n_i}\right) + \sum_{i=1}^N \mathcal{O}\left(d\sqrt{n_i\log n_i\log(n_i/\delta)}\right)
\end{align*}
with probability at least $1-4\delta-\sum_{i}\delta_i$. Substituting $n_i = 2^{i-1} n_1$ and $\delta_i = \frac{\delta}{2^{i-1}}$, we get
\begin{align*}
    R(n) \leq  (c_1 \| \theta \| + c_2) \sum_{i=1}^N \mathcal{O}\left( \sqrt{i\,dn_i\log n_1}\right) + \sum_{i=1}^N \mathcal{O}\left(\mathsf{poly}(i)\,d\sqrt{n_i\log n_1\log(n_1/\delta)}\right)
\end{align*}
with probability at least $1-4\delta -2\delta=1-6\delta$.
Using the above expression, we get the regret bound
\begin{align*}
    R(n) & \leq   \mathcal{O}\left((c_1 \| \theta \| + c_2) \sqrt{d\log n_1}+d\sqrt{\log n_1\log(n_1/\delta)} \right) \sum_{i=1}^N \mathsf{poly}(i)\sqrt{n_i}\\
    & \leq   \mathcal{O}\left((c_1 \| \theta \| + c_2)\sqrt{d\log n_1}+d\sqrt{\log n_1\log(n_1/\delta)} \right) \mathsf{poly}(N)\,\, \sum_{i=1}^N  \sqrt{n_i} \\
    & \leq    \mathcal{O}\left((c_1 \| \theta \| + c_2) \sqrt{d\log n_1}+d\sqrt{\log n_1\log(n_1/\delta)} \right) \mathsf{polylog}(n/n_1)\,\, \sum_{i=1}^N  \sqrt{n_i} \\
    & \leq   \mathcal{O}\left((c_1 \| \theta \| + c_2) \sqrt{d\log n_1}+d\sqrt{\log n_1\log(n_1/\delta)} \right) \mathsf{polylog}(n/n_1) \sqrt{n}\\
    & = \mathcal{O}\left(\left(\| \theta\| \sqrt{dn\log n_1}+d\sqrt{n\log n_1\log(n_1/\delta)} \right)\mathsf{polylog}(n/n_1) \right),
\end{align*}
where we have used that $N=\cO\left(\log (n/n_1)\right)$, and $\sum_{i=1}^N \sqrt{n_i} = \cO(\sqrt{n})$.
The above regret bound holds with probability greater than or equal to $1- 6\delta$, which completes the proof.
\end{proof}

%% file: app_cluster.tex
\section{CLUSTERING IN MULTI AGENT BANDITS: PROOFS OF RESULT}
\label{sec:app-cluster}

\begin{proof}[Proof of Lemma \ref{lem:cluster}]
Let us look at the parameter estimate of agent $i$, and without loss of generality, assume that agent $i$ belongs to cluster $j$. Since we let the agents play \texttt{OFUL} for $n$ time steps, from \cite{oful}, for agent $i$, we obtain
\begin{align*}
    \|\Hat{\theta}^{(i)} - \theta^*_j \|_{\Bar{V}_n} \leq 2 \sqrt{d \log(n/\delta)},
\end{align*}
where $\Bar{V}_n = \sum_{t=1}^n x_{{A_t},t} x_{{A_t},t}^\top + \lambda I$, where $A_t$ is the action of agent $i$ at time $t$. The above holds with probability at least $1-\delta$. Furthermore, we assume that \texttt{OFUL} is run with  regularization parameter $\lambda$ chosen as $\mathcal{O}(1)$. Continuing, we obtain
\begin{align*}
    \sqrt{\lambda_{\min}(\Bar{V}_n)} \, \|\Hat{\theta}^{(i)} - \theta^*_j \| \leq 2 \sqrt{d \log(n/\delta)}.
\end{align*}
We now use Assumption~\ref{ass:high-prob}, with  $\lambda = \mathcal{O}(1)$. Using Weyl's inequality, we obtain, $\lambda_{\min}(\Bar{V}_n) > \frac{\gamma}{2}\sqrt{n} \sqrt{\log (d/\delta)}$, with probability at least $1-\delta$. With this we have,
\begin{align*}
    \|\Hat{\theta}^{(i)} - \theta^*_j \| \leq \frac{2\sqrt{2}}{\sqrt{\gamma} n^{1/4}} \frac{1}{\log(d/\delta)^{1/4}} \sqrt{d \log(n/\delta)} = \frac{1}{n^{1/4}}\,\frac{2\sqrt{2} \sqrt{d}}{\sqrt{\gamma}} \sqrt{\frac{\log(n/\delta)}{\log^{1/2}(d/\delta)}}.
\end{align*}
From the separation condition on $\Delta$ and the choice of threshold $\eta$, we obtain
\begin{align*}
    \|\Hat{\theta}^{(i)} - \theta^*_j \| \leq \frac{\eta}{2},
\end{align*}
with probability at least $1-2\delta$. We now consider $2$ cases:

\textbf{Case I: Agents $i$ and $i'$ belong to same cluster $j$:} In this setup we have
\begin{align*}
    \|\Hat{\theta}^{(i)} - \Hat{\theta}^{(i')}\| &\leq \|\Hat{\theta}^{(i)} - \theta^*_j\| + \|\Hat{\theta}^{(i')} - \theta^*_j\| \\
    & \leq \frac{\eta}{2} + \frac{\eta}{2}= \eta,
\end{align*}
with probability at least $1-4\delta$.

\textbf{Case II: Agents $i$ and $i'$ belong to different cluster $j$ and $j'$ respectively:} In this case we have
\begin{align*}
    \|\Hat{\theta}^{(i)} - \Hat{\theta}^{(i')}\| &= \|(\Hat{\theta}^{(i)} -\theta^*_j) + (\theta^*_{j'} - \Hat{\theta}^{(i')}) - (\theta^*_{j'} - \theta^*_j) \| \\
    & \geq \|\theta^*_j -\theta^*_{j'} \| - \|(\Hat{\theta}^{(i)} -\theta^*_j)\| - \|(\Hat{\theta}^{(i')} -\theta^*_{j'})\| \\
    & \geq \Delta - \frac{\eta}{2} - \frac{\eta}{2} = \eta,
\end{align*}
with probability at least $1-4\delta$, where we use the condition that $\Delta > 2\eta$.

From the above 2 cases, if we select the threshold to be $\eta = \frac{\Delta}{2}$, every pair of machines are correctly clustered with probability at least $1-4\delta$. Taking the union bound over all $\binom{N}{2}$ pairs, we obtain the result.

\paragraph{Regret of agent $i$:} We now characterize the regret of agent $i$. Since agent $i$ played \texttt{OFUL} for $n$ steps, from \cite{oful}, we obtain
\begin{align*}
     R_i \leq \Tilde{\mathcal{O}}\left( d\sqrt{n}  \right) \log(1/\delta_1)
 \end{align*}
 with probability at least $1-\delta_1$.

\end{proof}

%% file: app_high_probability.tex
\section{HIGH PROBABILITY LOWER BOUND}
\label{sec:high_probability}
We note that any optimistic algorithm chooses actions based on a high probability confidence set of the true $\theta$, namely
\begin{align*}
    \|\theta - \hat{\theta}\|_{V_n} \leq c\sqrt{\ln{n}}
\end{align*}
Naturally it would be very useful if we could get an estimate on the lower bound for $\lambda_{\min}(V_n)$ instead of what we have on $\lambda_{\min}(\mathbb{E}[V_n])$. More precisely we would like a theorem which suggests $\lambda_{\min}(V_n) \geq \sqrt{n}$ for all $n$ more than some $n_0$, where $n_0$ depends upon the specific algorithm chosen and the geometry of the action space with high probability given that the algorithmic regret is $O(\sqrt{n})$. 

For the time being this seems difficult in the present setup. However what we would like to emphasize is what such a result could establish. We illustrate two practical problems in the Applications sections one in Model Selection and the other in  Clustering. For the time being we should emphasize that a direct corollary of this would be that a good regret algorithm would result in a best arm identification, given that the arm set is diverse enough. In this section we aim to discuss Assumption~\ref{ass:high-prob}.
We show that under a technical assumption, we can prove a high probability variant of Theorem~\ref{thm:main-result} as well as supplement the high probability claim by experimental observations added in section \ref{sec:experiments}.

\begin{assumption}[Stability Assumption]
\label{asm:stability}
Let any algorithm $\pi$, satisfy, for all $k \in [n]$, the following:
\begin{align*}
 \lambda_{\max}\Big(\expect{\sum_{i=k+1}^n A_iA_i^T|\mathcal{F}_k} - \expect{\sum_{i=k+1}^n A_i A_i^T|\mathcal{F}_{k-1}}\Big)\leq C \textit{  almost surely} ,
\end{align*}
where $\{A_i\}_{i=1}^k$ are the actions selected by $\pi$, and $\mathcal{F}_k$ is a filtration such that $\{A_i\}_{i=1}^k$ are adapted to $\mathcal{F}_k$ and $C$ is some positive constant.  
\end{assumption}

We agree that bridging the gap between in-expectation and in-high-probability results is an important open direction. Stability is one such technical tool.

\paragraph{Example of stability assumption being satisfied} Consider $K$ armed bandit problem with arm means $\{\mu_i\}_{i=1}^K$. In this setting, standard algorithms like UCB or TS suffers instance dependent regret of $O(1)$ \footnote{In our notation of Big O, we suppress the poly-logarithmic factors.}. Thus the number of times any sub-optimal arm is played is at most a constant number of times in the entire horizon. This implies 
\begin{align*}
    \norm{\expect{\sum_{i=k+1}^n A_iA_i^T|\mathcal{F}_k} - \expect{\sum_{i=k+1}^n A_i A_i^T|\mathcal{F}_{k-1}}} \leq \sum_{i=k+1}^{O(1)}2\norm{A_k}^2~,
\end{align*}
for any $k$, and hence the stability assumption is trivially satisfied.\\

The assumption essentially implies that playing a random action in the middle of an algorithmic run will not affect the overall trajectory of the actions-played drastically. Under this assumption, we can show that $\lambda_{\min}(V_n) \geq \Omega(\sqrt{n})$ with high probability.
To do so, we shall use matrix version of the Azuma-Hoeffding Inequality \cite{tropp2012user, tropp2015introduction}. For completeness we present the result in the appendix (see Lemma \ref{thm:matrix_azuma}).
We shall first prove a minimum eigenvalue analogue of the matrix Azuma-Hoeffding Inequality.

\begin{corollary}Consider a finitely adapted sequence $\{\Delta_k\}$ of self adjoint matrices in dimension $d$ and a fixed sequence of matrices $\{A_k\}$ of self-adjoint matrices such that $\mathbf{E}[\Delta_k |\mathcal{F}_{k-1}] = 0$ and $\Delta_k^2 \preccurlyeq A_k^2$ \textit{almost surely}. Then for all $t\geq 0$, we have
\label{cor: matrix_azuma}
$$\mathbf{P}\{\lambda_{\min}(\sum_{k=1}^n \Delta_k) \leq -t\} \leq d e^{-\frac{t^2}{8\sigma^2}}$$
\end{corollary}

\begin{proof}
Note that
\begin{align*}
 \mathbf{P}\{\lambda_{\min}(\sum_{k=1}^n \Delta_k) \leq -t\}=\mathbf{P}\{-\lambda_{\min}(\sum_{k=1}^n \Delta_k) \geq t\}=\mathbf{P}\{\lambda_{\max}(-\sum_{k=1}^n \Delta_k) \geq t\},   
\end{align*}
as $\lambda_{\max}(-A) = - \lambda_{\min}(A)$. Now, the proof follows by applying Theorem~\ref{thm:matrix_azuma} to the sequence $\{-\Delta_k\}$.
\end{proof}

Now we are ready to state and prove the high probability version of Theorem \ref{thm:main-result}.

\begin{theorem}
\label{thm:high-probability}
Fix a $\delta \in (0,1]$ and $n \geq n_0$,. Then, under the hypothesis of Theorem \ref{thm:main-result} and Assumption \ref{asm:stability}, with probability at least $1-\delta$, 
\begin{align*}
  \lambda_{\min}(V_n) \geq \gamma\sqrt{n} - (2+C)\sqrt{8n\ln{\frac{d}{\delta}}}.
 \end{align*}
Furthermore, if $2+C < \frac{\gamma}{\ln{d}}$, then for any $\delta \in (de^{-\frac{\gamma^2}{8(2+C)^2}}, 1)$, there exists a positive constant $\gamma_2$ such that
\begin{align*}
    \lambda_{\min}(V_n) \geq \gamma_2 \sqrt{n}
\end{align*}
with probability more than $1-\delta$. The constant $\gamma_2$ can be calculated as $\gamma_2 = \gamma - (2+C)\sqrt{8\ln{\frac{d}{\delta}}}$.
\end{theorem}

\begin{proof}
Let us choose the filtration $\mathcal{F}_k = \sigma(A_1,\cdots,A_k)$ for $k=1,\cdots,n$, as the natural filtration associated with the action-sequence $\{A_i\}_{i=1}^k$ and define the martingale difference sequence $\{\Delta_k\}_{k=1}^n$ as
\begin{align*}
\Delta_k = \mathbb{E}[\sum_{i=1}^nA_iA_i^\top|\mathcal{F}_k] - \mathbb{E}[\sum_{i=1}^nA_iA_i^\top|\mathcal{F}_{k-1}]    
\end{align*}
Note that, by the stability property of conditional expectation, we obtain
\begin{align*}
   \Delta_k &= \mathbb{E}[\sum_{i=1}^nA_iA_i^\top|\mathcal{F}_k] - \mathbb{E}[\sum_{i=1}^nA_iA_i^\top|\mathcal{F}_{k-1}] 
   \\ &= A_kA_k^\top - \mathbb{E}[A_kA_k^\top|\mathcal{F}_{k-1}] + \mathbb{E}[\sum_{i=k+1}^nA_iA_i^\top|\mathcal{F}_k] - \mathbb{E}[\sum_{i=k+1}^nA_iA_i^\top|\mathcal{F}_{k-1}],
\end{align*}
Therefore, we get by triangle inequality
\begin{align*}
    \|\Delta_k\| \leq \norm{A_kA_k^\top - \mathbb{E}[A_kA_k^\top|\mathcal{F}_{k-1}]} + \norm{\mathbb{E}[\sum_{i=k+1}^nA_iA_i^\top|\mathcal{F}_k] - \mathbb{E}[\sum_{i=k+1}^nA_iA_i^\top|\mathcal{F}_{k-1}]}
     \leq 2 + C
\end{align*}
where the last inequality follows by by assumption on the norm of the arms and the stability assumption \ref{asm:stability}. Therefore, by orthogonality of the martingale difference sequences, we have
\begin{align*}
    \|\Delta_k^2\| = \|\Delta_k\|^2 \leq (2+C)^2 \triangleq D^2,
\end{align*}
where we define the quantity $(2+C)$ as $D$.
Thus we have $\Delta_k^2 \preccurlyeq D^2\mathbb{I}_{d\times d}$ almost surely, where $\mathbf{I}_{d\times d}$ is the $d$-dimensional identity element and from the definition of $\sigma$ in Theorem \ref{thm:matrix_azuma}, $\sigma^2 = \|\sum_{i=1}^n D^2\mathbf{I}_{d\times d}\| = nD^2$.

Furthermore, note that $\sum_{k=1}^n \Delta_k = \sum_{i=1}^nA_iA_i^\top - \mathbb{E}[\sum_{i=1}^nA_iA_i^\top] \triangleq V_n - \mathbb{E}[V_n]$.
Then, using Corollary \ref{cor: matrix_azuma}, we have for any $t\geq 0$, that
\begin{align*}
 \mathbf{P}\{\lambda_{\min}(\sum_{k=1}^n \Delta_k) \leq -t\} = \mathbf{P}\{\lambda_{\min}(V_n - \mathbb{E}[V_n]) \leq -t\} \leq d e^{-\frac{t^2}{8nD^2}}\;.  
\end{align*}
Let us choose $t \geq \sqrt{8nD^2\ln{\frac{d}{\delta}}}$ for any $\delta \in (0,1)$.
Then, we have
\begin{align*}
    \lambda_{\min}(V_n - \mathbb{E}[V_n]) \geq - \sqrt{8nD^2\ln{\frac{d}{\delta}}}
\end{align*}
with probability more than $1-\delta$.
From Weyl's inequality (Lemma \ref{lemma:Weyl's}), we now obtain
\begin{align*}
    \lambda_{\min}(V_n) - \lambda_{\min}(\mathbb{E}[V_n]) = \lambda_{\min}(V_n) + \lambda_{\max}(-\mathbb{E}[V_n]) \geq \lambda_{\min}(V_n - \mathbf{E}[V_n]) \geq -\sqrt{8nD^2\ln{\frac{d}{\delta}}}
\end{align*}
with probability more than $1-\delta$.
Thus we get,
\begin{align*}
  \lambda_{\min}(V_n) \geq \lambda_{\min}(\mathbf{E}[V_n]) - \sqrt{8nD^2\ln{\frac{d}{\delta}}}  
\end{align*}
 with probability more than $1-\delta$.
Now, From Theorem \ref{thm:main-result} we have $\lambda_{\min}(\mathbb{E}[V_n]) \geq \gamma\sqrt{n}$ for some constant $\gamma > 0$ and for all $n \geq n_0$ for some $n_0$ which depends on $\gamma$ and the algorithmic constant $c$. This gives us 
\begin{align*}
  \lambda_{\min}(V_n) \geq \gamma\sqrt{n} - \sqrt{8nD^2\ln{\frac{d}{\delta}}}  \end{align*}
with probability more than $1-\delta$ for all $n \geq n_0$ hence proving the first part of the theorem.

For the second part of the theorem, assume $2+C < \frac{\gamma}{\ln{d}}$ then we have $d e^{\frac{\gamma^2}{8(1+C)^2}} < 1$. Thus we can find a $\delta$ in the open interval $(de^{\frac{\gamma^2}{8(1+C)^2}}, 1)$. Choosing such a $\delta$, we see $\gamma_2 \triangleq \gamma - \sqrt{8D^2\ln{\frac{d}{\delta}}} > 0$. Thus we have $$\lambda_{\min}(V_n) \geq \gamma_2 \sqrt{n}$$ with probability more than $1-\delta$ for all $n \geq n_0$

\end{proof}

%% file: app_tech_lemmas.tex
\section{TECHNICAL LEMMAS}
\label{sec:tech_lemma}

\begin{lemma}[Information Inequality] \cite{kaufmann2016complexity}
\label{lemma:Garivier_Kauffmann}
Let $\theta$ and $\theta'$ be two bandit parameters with policy induced measures $\mathbb{P}$ and $\mathbb{P'}$. Then for any measurable $Z \in (0,1)$, we have
\begin{align*}
    \mathrm{KL}(\mathbb{P} || \mathbb{P'}) \geq \mathrm{KL}(\mathrm{Ber}(\mathbb{E}_{\theta}[Z]) || \mathrm{Ber}(\mathbb{E}_{\theta'}[Z]))
\end{align*}
\end{lemma}

\begin{lemma}[Divergence Decomposition]\cite{lattimore2017end, lattimore2020bandit}
\label{lemma:Divergence_Decomposition}
Let $\mathbb{P}$ and $\mathbb{P'}$ be the action-observation sequence for a fixed bandit policy $\pi$ interacting with a linear bandit with standard Gaussian noise and parameters  
$\theta$ and $\theta'$ respectively.We have
\begin{align*}
    \mathrm{KL}(\mathbb{P} || \mathbb{P'}) = \frac{1}{2}\mathbb{E}_\theta\bigg[\sum_{t=1}^n\langle A_t, \theta - \theta'\rangle^2\bigg] =  \frac{1}{2}\|\theta - \theta'\|^2_{\bar{G}_n}
\end{align*}
\end{lemma}

\begin{lemma}[Davis-Kahan $\sin{\theta}$ theorem] \cite{stewart1990matrix}
\label{lemma:Davis_Kahan}
Let $A$ and $H$ be two symmetric $d\times d$ matrices.
Define $\Tilde{A} = A + H$
Let the spectral decomposition of $A$ and $\Tilde{A}$ be $A = \sum_{i=1}^d \lambda_iu_iu_i^T$ and $\Tilde{A} = \sum_{i=1}^d \Tilde{\lambda}_i\Tilde{u}_i\Tilde{u}_i^T$, respectively,
where $\lambda_1 \gg \lambda_2 \geq \cdots \geq \lambda_d$ and $\Tilde{\lambda}_1 \gg \Tilde{\lambda}_2 \geq \cdots \geq \Tilde{\lambda}_d$.
Then 
\begin{align*}
    \|u_1^T \begin{bmatrix}
\Tilde{u}_2 & \Tilde{u}_3 & \cdots & \Tilde{u}_d\\
\end{bmatrix}\|_2 \leq \frac{\|H\|}{\delta},
\end{align*}
where $\delta$ is the eigenvalue separation between $\lambda_1$ and $\Tilde{\lambda}_2,\cdots,\Tilde{\lambda}_d$. 
\end{lemma}

\begin{lemma}[Weyl's Inequality]\cite{stewart1990matrix}
\label{lemma:Weyl's}
For symmetric $A$ and $H$ we have 
\begin{align*}
    \lambda_i(A+H)\leq \lambda_i(A) + \lambda_{\max}(H)
\end{align*}
\end{lemma}

\begin{theorem}[Matrix Azuma]\cite{tropp2012user, tropp2015introduction}
\label{thm:matrix_azuma}
Consider a finitely adapted sequence $\{\Delta_k\}$ of self adjoint matrices in dimension $d$ and a fixed sequence of matrices $\{A_k\}$ of self-adjoint matrices such that $\mathbf{E}[\Delta_k |\mathcal{F}_{k-1}] = 0$ and $\Delta_k^2 \preccurlyeq A_k^2$ \textit{almost surely}. Then for all $t\geq 0$, we have
\begin{align*}
 \mathbf{P}\{\lambda_{\max}(\sum_{k=1}^n \Delta_k) \geq t\} \leq d e^{-\frac{t^2}{8\sigma^2}}~,   
\end{align*}
where $\sigma^2 = \|\sum_{k=1}^n A_k^2\|$.
\end{theorem}